\documentclass[nohyperref]{article}

\usepackage{microtype}
\usepackage{graphicx}
\usepackage{subfigure}
\usepackage{booktabs} 

\usepackage{hyperref}

\usepackage[accepted]{icml2022}

\usepackage{amsthm}

\usepackage{hyperref}
\usepackage{url}
\usepackage{float}
\usepackage{mathtools}
\usepackage{epsfig}
\usepackage{verbatim}
\usepackage{amssymb}
\usepackage{dsfont}
\usepackage{amsbsy}
\usepackage{amsmath}
\usepackage{bbm}
\usepackage{multirow}
\usepackage[capitalize,noabbrev]{cleveref}

\theoremstyle{plain}
\newtheorem{theorem}{Theorem}[section]
\newtheorem{proposition}[theorem]{Proposition}
\newtheorem{lemma}[theorem]{Lemma}
\newtheorem{corollary}[theorem]{Corollary}
\theoremstyle{definition}
\newtheorem{definition}[theorem]{Definition}

\theoremstyle{remark}
\newtheorem{remark}[theorem]{Remark}

\usepackage[textsize=tiny]{todonotes}
\allowdisplaybreaks[3]

\icmltitlerunning{When Does A Spectral Graph Neural Network Fail in Node Classification?}

\begin{document}
\onecolumn

\icmltitle{When Does A Spectral Graph Neural Network Fail in Node Classification?}




\begin{icmlauthorlist}
\icmlauthor{Zhixian Chen}{yyy}
\icmlauthor{Tengfei ma}{comp}
\icmlauthor{Yang Wang}{yyy}
\end{icmlauthorlist}

\icmlaffiliation{yyy}{Department
of Mathematics, Hong Kong University of Science and Technology, Hong Kong SAR,
China}
\icmlaffiliation{comp}{IBM T. J. Watson Research Center, New York, USA}

\icmlcorrespondingauthor{Zhixian Chen}{zchencz@connect.ust.hk}
\icmlcorrespondingauthor{Tengfei Ma}{Tengfei.Ma1@ibm.com}


\vskip 0.3in



\printAffiliationsAndNotice{}  

\begin{abstract}
 Spectral Graph Neural Networks (\textrm{GNNs}) with various graph filters have received extensive affirmation due to their promising performance in graph learning problems. However, it is known that \textrm{GNNs} do not always perform well. Although graph filters provide theoretical foundations for model explanations, it is unclear when a spectral \textrm{GNN} will fail. 
In this paper, focusing on node classification problems, we conduct a theoretical analysis of spectral \textrm{GNNs} performance by investigating their prediction error. With the aid of graph indicators including homophily degree and response efficiency we proposed, we establish a comprehensive understanding of complex relationships between graph structure, node labels, and graph filters. We indicate that graph filters with low response efficiency on label difference are prone to fail. To enhance \textrm{GNNs} performance, we provide a provably better strategy for filter design from our theoretical analysis - using data-driven filter banks, and propose simple models for empirical validation. Experimental results show consistency with our theoretical results and support our strategy.
\end{abstract}

\section{Introduction}
Graph Neural Networks (\textrm{GNNs}) have continuously attracted interest as their promising performance in various graph learning problems. It is known that most of \textrm{GNNs} are intrinsically graph filters \citep{kipf2016semi,defferrard2016convolutional,DBLP:journals/pieee/OrtegaFKMV18, nt2019revisiting}. With the theoretical foundation of filters, there is an increasing attempt at model explanation, e.g. explaining the behavior of various \textrm{GNNs} in node classification. \cite{nt2019revisiting} investigated the superiority of low-pass filters backed up with theoretical arguments while recent research \citep{balcilar2020analyzing,chang2020spectral,bo2021beyond} empirically revealed the weakness of \textrm{GNNs} with only low-pass filters in certain datasets. These contradictory views bring us to a question: why does a graph filter work on a dataset but not on another? More general, \textit{when does a graph filter fail} and \textit{what limits its prediction performance?} 


Existing theoretical research is mostly restricted to the investigation of filters themselves, such as exploring their expressive power \citep{oono2019graph,balcilar2020analyzing}, without taking their inconsistency of performance on different graphs into account. 
In this paper, we conduct a theoretical analysis of spectral \textrm{GNNs} performance by investigating their prediction errors on different graphs. Our preliminary result of prediction errors encourages us to have a comprehensive understanding of the complex and ambiguous relationships between graph structure, node labels, and graph filters. In Sect.\ref{sec.graph indicator}, we propose significant graph indicators including \textit{interaction probability} (as a metric of homophily) and \textit{response efficiency}. With the aid of them, we perform further analysis on prediction error which underpins deep insights of the failure of graph filters: 1. A graph filter fails when it has \textbf{low response efficiency} on label difference or input difference; 2. A graph filter is a hidden structure-adjustment mechanism and it fails when it is limited to make \textbf{graph homophilic high} enough by strengthening internal connections of classes; 3. Graph filters are prone to fail on graphs with \textbf{low information label differences}. It leads us to another question: \textit{how to design filters to improve \textrm{GNNs} performance?}

To address this concern, we apply our theoretical results to typical graph filters and investigate their potential behavior on different graphs. We show that low-pass filters are superior to high-pass filters in homophilic graphs and high-order filters have an advantage over low order filters in most cases. In addition, we provide a theoretical demonstration of the \textbf{superiority of filter banks} which have only been empirically used in previous works to enhance \textrm{GNNs} performance\cite{min2020scattering,gao2021message}. Based on these explorations, we propose an effective strategy for filter design, that is, learning filter banks in a data-driven manner.

 To verify the effectiveness of the strategy we proposed, we develop a simple framework, named \textrm{DEMUF}, to learn data-specified filter banks efficiently and examine our models on various datasets. Experimental results show that our model achieves a significant performance improvement compared with spectral \textrm{GNN} baselines across most benchmarks and have strong consistency with our theoretical conclusions. 

The rest of the paper is organized as follows: we formulate the prediction error of spectral \textrm{GNNs} and obtain a general lower bound in Sect.\ref{Sec.definition}. In Sect.\ref{sec.graph indicator}, we propose two groups of graph indicators which underpins our deep insights of graph filters from the spatial and spectral perspectives in Sect.\ref{Sec.filter analysis }. Following that, we draw two main conclusions in Sect.\ref{subsec.error analysis} and apply them to three types of filters in Sect.\ref{subsec.typical filters}. In Sect.\ref{Sec.model}, we develop a simple framework to implement the strategy we proposed for filter design and empirically validate our theoretical analysis.

\section{Related Work}
In this paper, we focus on the analysis of the performance of \textrm{GNNs} from the spectral perspective. Since \citet{bruna2014spectral} defined spectral graph filters and extended convolutional operations to graphs, various spectral graph neural networks have been developed. For example, \textrm{ChebNet} \citep{defferrard2016convolutional} defines the \textrm{Chebyshev} polynomial filter which can be exactly localized in the k-hop neighborhood. \citet{kipf2016semi} simplified the Chebyshev filters using a first-order approximation and derived the well-known graph convolutional networks (\textrm{GCNs}). \citet{bianchi2021graph} proposed the rational auto-regressive moving average graph filters (\textrm{ARMA}) which are more powerful in modeling the localization and provide more flexible graph frequency response, however more computationally expensive and also more unstable. Recently, \citet{min2020scattering} augmented conventional \textrm{GCNs} with geometric scattering transforms which enabled second-order filtering of graph signals and alleviated the over-smoothing issue. In addition, most graph neural networks originally defined in the spatial domain are also found essentially connected to the spectral filtering \citep{balcilar2020analyzing}. By bridging the gap between spatial and spectral graph neural networks, \citet{balcilar2020analyzing} further investigated the expressiveness of all graph neural networks from their spectral analysis. However, their analysis is limited to the spectrum coverage of a graph filter itself and lacks deeper insights into the graph-dependent performance of these filters.

Another related topic is the graph homophily/heterophily. One important graph indicator we propose in the paper is the homophily degree which we define through the interaction probability. Beyond that, there have been some other heuristic metrics for homophily/heterophily in previous works. \citet{pei2020geom} defined a node homophily index to characterize their datasets and help explain their experimental results for \textrm{Geom\_GCN}. \citet{zhu2020beyond} defined edge homophily ratio instead and identified a set of key designs that can boost learning from the graph structure in heterophily.
Some recent works analyzed the impact of heterophily on the performance of \textrm{GNNs} \cite{zhu2020beyond,jin2021universal,ma2021homophily}, but they are either limited to empirical study\cite{zhu2020beyond,jin2021universal} or just focused on \textrm{GCNs} \cite{ma2021homophily}. 
Our work differs from these works in that our homophily degree definition is only used as one of the graph indicators in our theoretical analysis of \textrm{GNNs} performance, and our analysis is also not limited to \textrm{GCNs} but for all spectral \textrm{GNNs}.

\section{Theoretical Analysis of Prediction Error}\label{Sec.definition}

\subsection{Problem Formulation}\label{sec.prediction error}
\textbf{Notations.} Let $\mathcal{G}_n$ be an undirected graph with additional self-connection, $A\in \mathbb{R}^{n\times n}$ be the adjacency matrix and $L=D-A$ be the Laplacian matrix, where $D$ is a diagonal degree matrix with $D_{ii}=\sum_{j}A_{ij}$. We denote $\tilde{A}=D^{-\frac 12}AD^{-\frac 12}$ and $\tilde{L}=I-\tilde{A}$ as the symmetric normalized Laplacian. Let $(\lambda_i,\mathbf{u}_i)$ be a pair of eigenvalue and unit eigenvector of $\tilde{L}$, where $0=\lambda_0\leq \dots \leq\lambda_{n-1}\leq 2$. In graph signal processing (\textrm{GSP}), $\{\lambda_i\}$ and $\{\mathbf{u}_i\}$ are called frequencies and frequency components of graph $\mathcal G_n$.

In this paper, we are mainly interested in multi-class node classification problems on $\mathcal{G}_n$ with labels $\mathcal{T}=\{0,\dots,K-1\}$. For $\forall k \in \mathcal{T}$, we denote $\mathcal{C}_k$ as the set of nodes with label $k$ and introduce a label matrix $Y\in\mathbb{R}^{n\times K}=(\mathbf{y}_0,\dots,\mathbf{y}_{K-1})$, where $\mathbf{y}_k$ is the indicator vector of $\mathcal{C}_k$.
Let $R =Y^\top Y$, then $Y^\top \mathbbm{1}=diag(R)$ and $R_{k}=|\mathcal{C}_k|$.
The general formulation of the $l+1$-th layer of spectral \textrm{GNNs} is $X^{(l+1)} = \sigma (g(\tilde{L})X^{(l)}W^{(l+1)})$, here $g(\tilde L)$ is so-called the graph filter, $\sigma(\cdot)$ is an activation function and $W^{(l+1)}$ is a learnable matrix. In multi-class classification problems, $\sigma(\cdot)=\text{softmax}(\cdot)$. In this paper, we say $X^{(l)}W^{(l+1)}$ is the input of $g(\tilde{L})$ in $l+1$-th layer.

\subsection{Prediction Error}
\begin{definition}[Prediction error] For a graph $\mathcal G_n$ with $\tilde L$, let $X=(\mathbf x_0,\dots,\mathbf x_{K-1})$ be the learnable input of $g(\tilde{L})$ in the last layer and $Y=(\mathbf y_0,\dots,\mathbf y_{K-1})$ be the label matrix, the prediction error is formulated as:
\begin{align}
&Er(X,Y)=\parallel \sigma(g(\tilde L)X)-Y\parallel_F^2=\sum_{l\in \mathcal T}Er(\mathbf x_l,\mathbf y_l),\\
&Er(\mathbf x_l,\mathbf y_l)=\parallel \frac {e^{g(\tilde L)\mathbf x_l}}{e^{g(\tilde L)\mathbf x_l}+\sum_{k\ne l}e^{g(\tilde L)\mathbf x_k}}-\mathbf y_l\parallel_2^2.
\end{align}
\end{definition}

Since $Er(X,Y)=\sum_{l\in \mathcal T}Er(\mathbf x_l,\mathbf y_l)$, estimating the entire prediction error equals to estimate that of any single label. Without loss of generality, in the following discussion, we focus on label $\mathbf y_0$ and investigate $Er(\mathbf x_0,\mathbf y_0)$. By denoting $\mathbf y'_0=\mathbf y_0$ and $\mathbf y'_1=\sum_{l=1}^{K-1}\mathbf y_l$, we obtain a corresponding binary classification problem where $Er(\mathbf x'_0,\mathbf y'_0)$ is an approximation of $Er(\mathbf x_0,\mathbf y_0)$.   For simplicity, in the rest of paper, we investigate $Er(\mathbf x_0,\mathbf y_0)$ in a binary classification. The theorem below provides a lower bound of it.

\begin{theorem}[Prediction error]\label{thm.lower bound}
In a binary classification problem with label matrix $Y=(\mathbf y_0, \mathbf y_1)$. Let $X=(\mathbf x_0, \mathbf x_1)$ be the input matrix, we have:
\begin{align*}
    Er(\mathbf x_0,\mathbf y_0)\ge&\frac n4-\frac {(\mathbf y_1-\mathbf y_0)^\top \psi(\mathbf z)}4+\frac {\parallel \psi(\mathbf z)\parallel_2^2}{16}\\
    -&\frac {\parallel \psi(\mathbf z)\parallel_3^3}{48}-\frac {\parallel \psi(\mathbf z)\parallel_4^4}{96}-\frac C{(1+e)^2}\\
    >&\frac {167}{800}n-\frac 14 \sum_i\psi\big((\mathbf y_{1i}-\mathbf y_{0i})\mathbf z_i\big)
\end{align*}
where $\mathbf z=g(\tilde L)(\mathbf x_{1}-\mathbf x_{0})$, $\psi(x)=\min\{\max\{x,-1\},1\}$ is a clamp function limiting $x$ to $[-1,1]$ and $C$ is cardinality of $\mathcal S_{\mathbf y,\mathbf z}=\{(\mathbf y_{0i},\mathbf z_{i})|\mathbf z_{i}<-1,\mathbf y_{0i}=1\; \text{or}\;\mathbf z_{i}>1,\mathbf y_{0i}=0\}$, i.e., $C=|\mathcal S_{\mathbf y,\mathbf z}|$.
\end{theorem}
    
The proof is in Appendix \ref{AP.proof lower bound}. The first inequality holds when $\mathbf x_0=\mathbf x_1$ where $Er(\mathbf x_0,\mathbf y_0)=\frac{n}{4}$. This theorem indicates that prediction errors of binary classification are dependent on graph filter $g(\tilde L)$, $\Delta \mathbf y=\mathbf y_0-\mathbf y_1$ and $\Delta \mathbf x=\mathbf x_0-\mathbf x_1$, what we refer to as \textit{label difference} and \textit{input difference} in this paper, respectively. Then we have
\begin{align}\label{eq.lower bound}
 Er(\mathbf x_0,\mathbf y_0)>\frac {167}{800}n-\frac 14 \sum_i\psi\big(\Delta \mathbf y_{i}(g(\tilde L)\Delta \mathbf x)_i\big).
\end{align}
 It makes us aware of the need to fully understand their impact on the performance of spectral \textrm{GNNs}.
\section{Proposed Graph Indicators}\label{sec.graph indicator}
In this section, we explore the relationships between graph structure, graph signals including node labels, and graph filters and define related graph indicators which underpins our further investigation on prediction error in Sect.\ref{Sec.filter analysis }.

\subsection{Spatial Graph Indicators}
Homophily of graphs is considered an indisputable common property of most graphs. In this section, we derive a measure of homophily degree from interaction probability.

For a random walk on $\mathcal{G}_n$, $P^k=(D^{-1}A)^k$ is the $k$-step transition matrix where $P^k_{ij}$ is the probability that a random walker starting from node $v_i$ arrives at $v_j$ after $k$ steps. $\sum_{j\in\mathcal{C}_l}P_{ij}^k$ is the probability that a random walker starting from $v_i$ stays in $\mathcal{C}_l$ at the $k$-th step, demonstrating the relative preference/closeness of node $v_i$ for $\mathcal{C}_l$ at $k$-th scale. Based on this, we define interaction probability to reflect the strength of interaction between two classes.

\begin{definition}[$k$-step interaction probability]\label{def.interaction} 
For $l,m\in \mathcal{T}$, denote $P^k=(D^{-1}A)^k$, the $k$-step interaction probability matrix is formulated as: 
\begin{align}
&\Pi^{k}_{lm}=\frac 1{R_l}\sum\limits_{v_i\in\mathcal{C}_l,v_j\in\mathcal{C}_m} P^k_{ij}=\frac {\mathbf{y}_l^\top P^k\mathbf{y}_m}{\mathbf{y}_l^\top\mathbf{y}_l}\\
&\Pi^{k}=(Y^\top Y)^{-1}Y^\top P^kY=R^{-1}Y^\top P^kY.
\end{align}
 $\Pi^{k}_{lm}$ is the probability that a random walker from $\mathcal{C}_l$ arrives at $\mathcal{C}_m$ after $k$ steps and $\sum_{m\in\mathcal T}\Pi_{lm}^k=1$.
\end{definition}

 Since $P$ is not symmetric, then $\Pi_{lm}^k\neq \Pi_{ml}^k$. Below we propose \textit{a symmetric variant of interaction probability}, which plays a key role in our theoretical analysis. 
 
 \begin{definition}[$k$-step symmetric interaction probability]\label{def.symmetric interaction}
For $l,m\in \mathcal{T}$, denote $\tilde{A}^k = (D^{-\frac 12}AD^{-\frac 12})^k$, the $k$-step symmetric interaction probability matrix is formulated as: 
\begin{align}
\tilde{\Pi}^k=R^{-\frac 12}Y^\top \tilde{A}^kYR^{-\frac 12}.
\end{align}
\end{definition}
 
\begin{remark}\label{remark.interaction} If $|\mathcal T|\neq n$, $(R^{-\frac 12}Y^\top \tilde AYR^{-\frac 12})^k\neq R^{-\frac 12}Y^\top \tilde{A}^kYR^{-\frac 12}$, i.e. $(\tilde \Pi)^k\neq \tilde \Pi^{k}$. Also, $(\Pi)^k\neq \Pi^{k}$. That is, $k$-step (symmetric) interaction probability is not the $k$-th power of $1$-step (symmetric) interaction probability.
\end{remark}

\textbf{Notations.} In the rest of paper, we take $\tilde \pi^k_l$ as the shorthand for $\tilde \Pi^k_{ll}$. We denote $\tilde{g}(\tilde \Pi)=R^{-\frac 12}Y^\top g(\tilde{A})YR^{-\frac 12}$ to avoid confusion with $g(\tilde\Pi)=g(R^{-\frac 12}Y^\top \tilde{A}YR^{-\frac 12})$.

 Noting that $\tilde{\Pi}^k\mathbbm{1}\neq \mathbbm{1}$, $\tilde{\Pi}^k$ is not a probability measure in the strict sense. However, it is a bridge to other graph indicators, as we will show in the rest of this section. Below, we show the relationship between $\tilde{\Pi}^k$ and $\Pi^k$.
 
\begin{proposition}[Interaction probability]\label{pro.interaction inequality}
For $l,m \in\mathcal{T}$, $
\tilde \pi_l^{2k}\ge(\tilde \pi_l^k)^2 \text{ and } \pi_l^k\ge\tilde{\pi}^k_l.$ More generally, $R_{l}\Pi^k_{lm}+R_{m}\Pi^k_{ml}\ge 2\sqrt{R_{l}R_{m}}\tilde{\Pi}^k_{lm}$.
\end{proposition}
The proof can be found in Appendix \ref{AP.proof interaction}. It indicates that $\tilde \pi_l^k$ is the lower bound of $\pi_l^k$. Recall the descriptive definition of graph homophily - nodes with the same labels are more likely to cluster together, for a homophilic graph with symmetric interaction probability $\tilde \Pi$, self-interaction probability $\tilde \pi_{l}^k$ is expected to gap away from $\{\tilde \Pi_{lm}^k|m\in\mathcal T,m\ne l\}$. 

\begin{definition}[$k$-homophily degree]\label{def.homophily}
For a graph $\mathcal G_n$ with symmetric interaction probability $\tilde \Pi$, the $k$-homophily degree of $\mathcal G_n$ is defined as
\begin{align}
&\mathcal H_k(\tilde \Pi|\mathcal C_l)=\sqrt{\frac {R_l}{n}}\tilde \pi_{l}^k-\sum_{m\ne l}\sqrt{\frac {R_m}{n}}\tilde \Pi_{lm}^k,\\
&\mathcal H_k(\tilde \Pi|\mathcal G_n)=\sum_l\sqrt{\frac {R_l}{n}}\mathcal H_k(\tilde \Pi^k|\mathcal C_l)\nonumber=\frac 1n\sum_l(R_l\tilde \pi_{l}^k-\sum_{m\ne l}\sqrt{R_mR_l}\tilde \Pi_{lm}^k),
\end{align}
where $\mathcal H_k(\tilde \Pi|\mathcal C_l)$ is the $k$-homophily degree of $\mathcal C_l$.
\end{definition}
\begin{remark}
$\mathcal H_k(\tilde \Pi|\mathcal C_l),\mathcal H_k(\tilde \Pi|\mathcal G_n)\in[-1,1]$. We say $\mathcal G_n$ is a $k$-homophilic graph if $\mathcal H_k(\tilde \Pi|\mathcal G_n)>0$.
\end{remark} 
In binary classification problems, $\mathcal H_k(\tilde \Pi|\mathcal G_n)=\frac 1n(R_0\tilde \pi_0^k+R_1\tilde \pi^k_1-2\sqrt{R_0R_1}\tilde \Pi^k_{01})$. Larger $\mathcal H_k(\tilde \Pi|\mathcal C_l)$ means $\mathcal C_l$ have denser internal connections and sparser connections between other classes. Then intuitively, \textit{$\mathcal H_k(\tilde \Pi|\mathcal G_n)$ reflects the possibility to obtain a node's label directly from its neighbors}. As we illustrate in Sect.\ref{subsec.analysis homophily}, one of the keys to the success of graph filters is to increase the possibility by strengthening internal connections of classes, that is, make graph more homophilic.

\subsection{Spectral Graph Indicators}\label{Sec.Frequency distribution}

In this section, we develop another indicator - \textit{repsonse efficiency}, to measure the effect of graph filters $g(\tilde L)$ applied to different graph signals. Recall that $\{\lambda_i\}$ and $\{\mathbf{u}_i\}$ are graph frequencies and frequency components. For a graph signal $\mathbf{x}$ with spectrum $\alpha=\{\alpha_i=\langle\mathbf{u}_i,\mathbf{x}\rangle\}$, we are able to represent it as $\mathbf{x}=\sum\alpha_i\mathbf{u}_i$. Noting that $\alpha_k^2\big/\sum_i\alpha_i^2$ shows the percentage of occurrences in $\mathbf{x}$ for each $\mathbf{u}_k$, we introduce a distributional representation of $\mathbf{x}$.

\begin{definition}[Frequency distribution] We define $\mathbf{f}$, the frequency of signal $\mathbf{x}$, as a random variable taking values in the set of graph frequencies with probability $\Pr(\mathbf{f}=\lambda_k)=\alpha_k^2\big/\sum_i\alpha_i^2$. The probability describes the frequency distribution of signal $\mathbf{x}$.
\end{definition}

Since probability $\Pr(\mathbf{f}=\lambda_k)=\alpha_k^2\big/\sum_i\alpha_i^2$ is the weight of $\mathbf{u}_k$ in $\mathbf{x}$ and $g(\lambda_k)$ reflects how filter $g(\tilde L)$ acts on $\mathbf{u}_k$, then we claim that $\sum_kg(\lambda_k)\Pr(\mathbf{f}=\lambda_k)$ is the effect of $g(\tilde L)$ acting on $\mathbf{x}$. Although an all-pass filter can pass all frequencies, we appreciate a filter with a high magnitude response for important frequencies so that we can capture the main frequency information of a signal $\mathbf{x}$ efficiently.
 \begin{definition}[Response Efficiency] For a graph filter $g(\tilde L)$ and a signal $\mathbf x$ with spectrum $\alpha$, the response efficiency of $g(\tilde L)$ on $\mathbf x$ is defined as
 \begin{align*}
     \mu_g(\mathbf x)= \frac {\sum_ig(\lambda_i)\alpha^2_i}{\big(\sum_{i}g(\lambda_i)\big)\big(\sum_i\alpha^2_i\big)}.
 \end{align*}
 It can also be denoted as $\mu_g(\alpha)$.
\end{definition}
We obtain a high response efficiency when the magnitude response of $g(\tilde L)$ is positively associated with the frequency distribution of $\mathbf x$. In Sect.\ref{subsec. analysis response efficiency}, we have further discussion of the relationship between graph response efficiency and its prediction performance. Before that, we investigate the consistency of proposed graph indicators.

\begin{proposition}\label{The.moment}
Let $\mathbf{f}_l$ be the frequency of label $\mathbf{y}_l$, for a graph filter $g(\cdot)$, we have $\mu_g(\mathbf y_l)=\frac {\big(\tilde{g}(I-\tilde{\Pi})\big)_{ll}}{\sum_ig(\lambda_i)}$. Specially, when $g=(\cdot)^n$, $\mu_g(\mathbf y_l)=\frac {\mathbb{E}[\mathbf{f}_l^n]}{\sum_ig(\lambda_i)}$.
\end{proposition}
The proof of this proposition can be found in Appendix \ref{AP.proof interaction}. Recall that $\tilde{g}(I-\tilde{\Pi})=R^{-\frac 12}Y^\top g(I-\tilde{A})YR^{-\frac 12}$, we have $\mathbb{E}[\mathbf{f}_l]=1-\tilde{\pi}_l$, $\mathbb{E}[\mathbf{f}_l^2]=1-2\tilde{\pi}_l+\tilde{\pi}_l^2$ and the variance of $\mathbf{f}_l$: $\text{Var}(\mathbf{f}_l)=\tilde{\pi}_l^{2}-(\tilde{\pi}_l)^2$. According to Proposition \ref{pro.interaction inequality}, $\tilde{\pi}_l\leq\pi_l\leq 1 $ and $\tilde{\pi}_l^{2}\ge(\tilde{\pi}_l)^2$. Therefore, when $\tilde{\pi}_l$ approaches 1, which reflects a high homophily degree of $\mathcal C_l$, both the mean and variance of label frequency are close to 0. It implies that, \textbf{for a highly homophilic graph, the main information of labels is low-frequency so it should be assigned to low-pass filters}. Rigorous support for this argument can be found in Sect.\ref{subsec.typical filters}.

\section{Analysis of Graph Filters}\label{Sec.filter analysis }

  In this section, we aim to figure out two major concerns: \textit{what causes the failure of a graph filter} and \textit{how to design filters to improve \textrm{GNNs} performance?} Precisely, we first provide a deep understanding of the performance of graph filters concerning label prediction with the aid of graph indicators we proposed above, then apply our theoretical conclusion to typical filters. From this, we have obtained insights into filter design.
  
\subsection{Analysis of Prediction Error}\label{subsec.error analysis}
 In this section, we focus on a family of polynomial filters $\mathcal S_g=\{g \text{ is polynomial }|g([0,2])\in[0,1],\sum_ig(\lambda_i)>1\}$. We denote $\mathcal I_g=\{i|g(\lambda_i)\ne 0, i=0,\dots,n-1\}$ as the indicator set of nonzero elements in $\{g(\lambda_i)\}$.

\subsubsection{From Filter Response Efficiency}\label{subsec. analysis response efficiency}

Let $ \delta,\eta$ be the spectra of $\Delta\mathbf y,\; \Delta\mathbf x$, respectively, it is trivial to revisit the inequality \eqref{eq.lower bound} of ${Er}(\mathbf x_0,\mathbf y_0)$ as: 
\begin{align*}
{Er}(\mathbf x_0,\mathbf y_0)>\frac {167}{800}n -\frac 14\sum_i\psi\big(g(\lambda_i)\delta_i\eta_i\big).
\end{align*}       
We attempt to establish an analysis of $\sum_i\psi\big(g(\lambda_i)\delta_i\eta_i\big)$ which is the critical term of this lower bound in the spectral domain. Before that, we introduce information content proposed by information theory to measure the informativeness of signal frequency. 
\begin{definition}[Information content] For a signal with spectrum $\delta$, we say $\mathbf I(\delta)=-\sum\limits_{i\in \mathcal I_{\delta}}\log \frac {\delta_i^2}{\sum_k\delta_k^2}$ is the information content of $\delta$ where $\mathcal I_{\delta}=\{i|\delta_i\ne 0, i=0,\dots,n-1\}$ is the indicator set of nonzero elements of $\delta$.
\end{definition}

\begin{theorem}\label{thm.upper bound 1}
Given a label difference $\Delta \mathbf y$ with spectrum $\delta$, for an arbitrary input difference $\Delta \mathbf x$ with spectrum $\eta$, for a graph filter $g(\cdot)\in\mathcal S_g$, we construct $\tilde \eta_i= \psi_{\frac 1{g(\lambda_i)\delta_i}}(\eta_i)$, where $\psi_{\frac 1{g(\lambda_i)\delta_i}}(x)=\min\{\max\{x,-\frac 1{g(\lambda_i)\delta_i}\},\frac 1{g(\lambda_i)\delta_i}\}$ is a clamp function such that $|\tilde \eta_ig(\lambda_i)\delta_i|\le 1$ and 
\begin{align}
    &\sum_{i=0}^{n-1}\psi(\eta_ig(\lambda_i)\delta_i)\le \frac 1{m_g}\min\{\mathcal M(g,\delta), \mathcal M(g,\tilde\eta)\},
\end{align}
where $m_g=\min\limits_{i\in \mathcal I_{g,\delta,\tilde \eta}}g(\lambda_i)$, $c(g,\delta)=\frac {\sum_{i\in \mathcal I_{g}-\mathcal I_{g,\delta}}g(\lambda_i)}{\sum_{i\in \mathcal I_{g,\delta}}g(\lambda_i)}$ and $\mathcal M(g,\delta)=\frac {-\mathbf I(\delta)}{\log (1+c(g,\delta))\mu_g(\delta)}$ with $\mathcal I_{g,\delta}=\mathcal I_{g}\cap\mathcal I_{\delta}$.
\end{theorem}

The theorem indicates that the \textbf{prediction error of a given graph filter is bounded by its response efficiency on labels and inputs difference}. In \textrm{GNNs}, inputs are learnable. Ideally, a \textrm{GNN} with filter $g$ can learn appropriate inputs with large enough $\mu_g(\eta)$. In this way, the prediction error will be restricted only by the graph structure, node labels, and the filter itself, that is what Corollary \ref{co.upper bound} illustrates.

\begin{corollary}[Spectral lower bound]\label{co.upper bound}
For a binary classification problem, with the same settings and notations in Theorem \ref{thm.upper bound 1}, we have
\begin{align}
     Er(\mathbf x_0,\mathbf y_0)
>\frac {167}{800}n +\frac {\mathbf I(\delta)}{m_g\log (1+c(g,\delta))\mu_g(\delta)}
\end{align}
\end{corollary}

For a given classification problem, we claim that a graph filter fails if the lower bound of its prediction error is large. Corollary \ref{co.upper bound} (proof is provided in Appendix \ref{Ap.proof filter bank}) impels us toward deep insights of graph filters in terms of response efficiency: 1. \textbf{A graph filter fails when it has low response efficiency on label difference}, i.e., small $\mu_g(\delta)$, which means that it can't capture the main information used for label identification efficiently; 2. \textbf{Most filters fail on graphs with low information label difference}, i.e., small $\mathbf I(\delta)$, which means that a closed difference of labels’ probability on different frequency components would have been difficult to distinguish.

\subsubsection{From Graph Homophily Degree}\label{subsec.analysis homophily}
Recall the attempt we made to explore the relation between filter response efficiency on labels and graph homophily degree at the end of Sect.\ref{sec.graph indicator}, it inspires us to explain graph filters from a perspective of graph homophily.

\begin{theorem}[Spatial lower bound]\label{thm.upper bound 2}
Given a binary classification problem on a graph $\mathcal G_n$ whose label difference is $\Delta \mathbf y$ with spectrum $\delta$, for a graph filter $g(\cdot)\in\mathcal S_g$ and arbitrary input $X$, we have
\begin{align}
Er(X,Y)
>\frac {167}{400}n +\frac {\mathbf I(\delta)}{2m_g\log\frac{\mathcal H_1(\tilde g(I-\tilde \Pi)|\mathcal G_n)}{\sum_ig(\lambda_i)}},
\end{align}
where $m_g=\min\limits_{i\in \mathcal I_{g,\delta}}g(\lambda_i)$, $\tilde{g}(I-\tilde \Pi)=R^{-\frac 12}Y^\top g(I-\tilde{A})YR^{-\frac 12}$ and $\mathcal H_1(\tilde g(I-\tilde \Pi)|\mathcal G_n)=\frac {R_0}n(\tilde g(I-\tilde \Pi))_{00}+\frac {R_1}n\tilde g(I-\tilde \Pi))_{11}-2\frac {\sqrt {R_0R_1}}n\tilde g(I-\tilde \Pi))_{01}$.
\end{theorem}

The proof of Theorem \ref{thm.upper bound 2} provided in Appendix \ref{append.proof} shows that for any function $g(\cdot)$ which is nonnegative on the closed interval $[0,2]$ , $\mathcal H_1(\tilde g(I-\tilde \Pi)|\mathcal G_n)\in[0,1]$. 

This theorem brings us an explanation of graph filters in terms of homophily: \textit{a graph filter is a hidden structure-adjustment mechanism} which transforms graph structure to $g(I-\tilde A)$. It obtains poor prediction performance when it \textbf{fails to make the high homophily degree $\mathcal H_1(\tilde g(I-\tilde \Pi)|\mathcal G_n)$ of transformed graph high enough.} From this, we are able to glimpse the prediction capacity of filters on a given graph through the modified homophily degree.

\subsection{Applied to Specific Graph Filters}\label{subsec.typical filters}

In this section, we apply the above observations to specific graph filters and provide practical strategies to enhance \textrm{GNNs} performance. Here, we consider typical filters including first/second-order low/high-pass filters and investigate their advantage/disadvantage on different graphs by comparing their $\mathcal H_1(\tilde g(I-\tilde \Pi)|\mathcal G_n)$.

 \textbf{Notations.} We denote $\mathcal S^1_g=\{g(\tilde L)=\epsilon_1I+\epsilon_2\tilde L|\epsilon_2\ne 0, g([0,2])\in[0,1],\sum_ig(\lambda_i)>1\}$ as a family of first-order graph filters and $\mathcal S^2_g=\{g(\tilde L)=\epsilon'_1I+\epsilon'_2\tilde L+\epsilon'_3\tilde L^2|\epsilon'_3<0, g([0,2])\in[0,1],\sum_ig(\lambda_i)>1\}$ as the family of second-order graph filters. We say $g_1$ with $\epsilon_2<0$ is a low-pass filter and $g_1$ with $\epsilon_2>0$ is a high-pass filter. Similarly, since $g_2(\tilde L)=\epsilon'_3(\frac {\epsilon'_2}{2\epsilon'_3}I+\tilde L)^2+(\epsilon_1'-\frac {\epsilon_2'^2}{4\epsilon_3'})I$, we say $g_2$ with $-\frac {\epsilon'_2}{2\epsilon'_3}<1$, i.e., $\epsilon'_2+2\epsilon'_3<0$ is a low-pass filter and $g_2$ with $\epsilon'_2+2\epsilon'_3>0$ is a high-pass filter.

\begin{theorem}[Low/high-pass]\label{thm.low/high fitler}
Given a graph $\mathcal G_n$ with interaction probability $\tilde \Pi$, let $g_1=\arg\max_{g\in\mathcal{S}_g^1}\mathcal H_1(\tilde g(I-\tilde \Pi)|\mathcal G_n)$ and $g_2=\arg\max_{g\in\mathcal{S}_g^2}\mathcal H_1(\tilde g(I-\tilde \Pi)|\mathcal G_n)$, 
\begin{itemize}
    \item when $\mathcal H_1(\tilde \Pi|\mathcal G_n)>0$, $g_1, g_2$ must be low-pass,
    \item when $\mathcal H_1(\tilde \Pi|\mathcal G_n)<0$, $g_1, g_2$ must be high-pass.
\end{itemize}
Moreover, denote $A=\frac {(\sqrt{R_0}-\sqrt{R_1})^2}n$, we have
\begin{align*}
\mathcal H_1(\tilde g_1(I-\tilde \Pi)|\mathcal G_n)&=\frac A2\int_{0}^2g_1d\lambda+|\epsilon_2\mathcal H_1(\tilde \Pi|\mathcal G_n)|,\\
\mathcal H_1(\tilde g_2(I-\tilde \Pi)|\mathcal G_n)&=\frac A2 \int_{0}^2g_2d\lambda+\epsilon'_3(\mathcal H_2(\tilde \Pi|\mathcal G_n)-\frac A3)+|(\epsilon'_2+2\epsilon'_3)\mathcal H_1(\tilde \Pi|\mathcal G_n)|.
\end{align*}
\end{theorem}
This theorem rigorously validates our inference of Proposition \ref{The.moment}: \textbf{low-pass filters
are superior to high-pass filters on homophilic graphs.} Compared with first-order filters, second-order filters involve 2-homophily degrees so that they can identify more different graphs. Actually, we show that second-order filters are easier to have better performance than first-order filters.

\begin{theorem}[First/second-order]\label{thm.first/second fitler}
Given a graph $\mathcal{G}_n$ with interaction probability $\tilde \Pi$ and $a=\frac {2(\sqrt{R_0}-\sqrt{R_1})^2}{3n}$, for $ g_1\in \mathcal S_g^1$, there exists a $g_2\in \mathcal S_g^2$ such that
\begin{enumerate}
    \item when $\mathcal H_2(\tilde \Pi|\mathcal G_n)\in[-1,a]$, $\mathcal H_1(\tilde g_2(I-\tilde \Pi)|\mathcal G_n)>\mathcal H_1(\tilde g_1(I-\tilde \Pi)|\mathcal G_n)$ ; 
    \item when $\mathcal H_2(\tilde \Pi|\mathcal G_n)\in[a,1]$, $|\mathcal H_1(\tilde g_2(I-\tilde \Pi)|\mathcal G_n)-\mathcal H_1(\tilde g_1(I-\tilde \Pi)|\mathcal G_n)|\le -\frac {\epsilon'_3}3\mathcal H_2(\tilde \Pi|\mathcal G_n)$. 
\end{enumerate}
\end{theorem}
It shows the superiority of second-order filters on most graphs apart from those with a high 2-homophily degree and gives a guarantee to narrow the disadvantage. There is an intuitive hyperthesis to be inferred from this theorem: \textbf{high-order filters have an advantage over low order filters in most cases.} However, high-order filters would bring high computation costs. In practice, to enhance the \textrm{GNNs} performance, filter banks are empirically used in previous works. Below, we provide a \textbf{rigorous demonstration of the advantage of filter banks}.

\begin{theorem}[Filter bank]\label{thm.filter bank1}
Given a graph $\mathcal G_n$ with normalized Laplacian matrix $\tilde L$ with eigenvalues $\{\lambda_i\}$, $\forall g_1(\cdot)$ in $\mathcal S_g^1$ , there exists a single filter $g_2(\cdot)\in\mathcal S_g^1$ and positive $l_1$, $l_2$ s.t. for $g=l_1g_1+l_2g_2$ we have
$$
\frac {m_{g_1}}{m_g}\log\frac {\mathcal H_1(\tilde g(I-\tilde \Pi)|\mathcal G_n)}{\sum_ig(\lambda_i)}> \log\frac {\mathcal H_1(\tilde g_1(I-\tilde \Pi)|\mathcal G_n)}{\sum_ig_1(\lambda_i)},
$$
and $l_1+l_2=1$, here $m_{g}=\min_{i\in \mathcal I_{g}}g(\lambda_i)$. 
\end{theorem}

This theorem can be generalized to the family of second-order graph filters $\mathcal S^2_g$. 

\subsection{A Strategy for Filter Design}\label{subsec.filter design}
According to our theoretical framework and conclusion on prediction error of a specific label, a graph filter that has a high response efficiency on this label difference or which can strengthen the internal connection of this class is the key to the success of \textrm{GNNs}. Considering a $K$-class classification problem, as we showed in Sect.\ref{sec.prediction error}, we will obtain $K$ related binary classification problems with $K$ label differences. When these label differences have a high diversity of frequency distributions, a low order single filter is prone to fail since it can barely have high response efficiencies on all of them. In these cases, it is necessary to leverage a filter bank to handle different frequency distributions.

\textbf{How to design a powerful filter bank efficiently?}
Theoretically, piling up sufficient numbers of graph filters to capture all the frequency components can improve prediction performance while it is very expensive. On the other hand, following our theoretical results, it is not difficult to design an ideal filter bank with the frequency distributions of label differences or graph interaction probability on hand. However, such graph indicators are usually unknown. Nevertheless, node features are accessible in attribute graphs and provide valid information for classification. As illustrated in Theorem \ref{thm.upper bound 1}, an ideal filter should have a high response frequency on both input difference and label difference. Here we assume that for each label there is some information provided by (partial) features useful for identification. Also, target frequency components and features may differ for different labels. When we classify labels with different target components having the same frequency, we need to separate objects with different target components before filtering. Otherwise, it is likely to bring noise and hurt the performance. Therefore, feature disentanglement is necessary. On the other hand, since a filter bank is applied to features to capture valid information, filter design should be targeted at specific graphs and features.

\section{Model and Empirical Study}\label{Sec.model}
As already emphasized, one practical strategy for improving prediction performance is to learn a filter bank in a data-driven manner. In this section, to verify this theoretical strategy, we propose a simple framework - disentangled multi-filter framework (\textrm{DEMUF})

\begin{figure*}[h!]
\centering 
\vskip 0.2in
\includegraphics[width=0.85\textwidth]{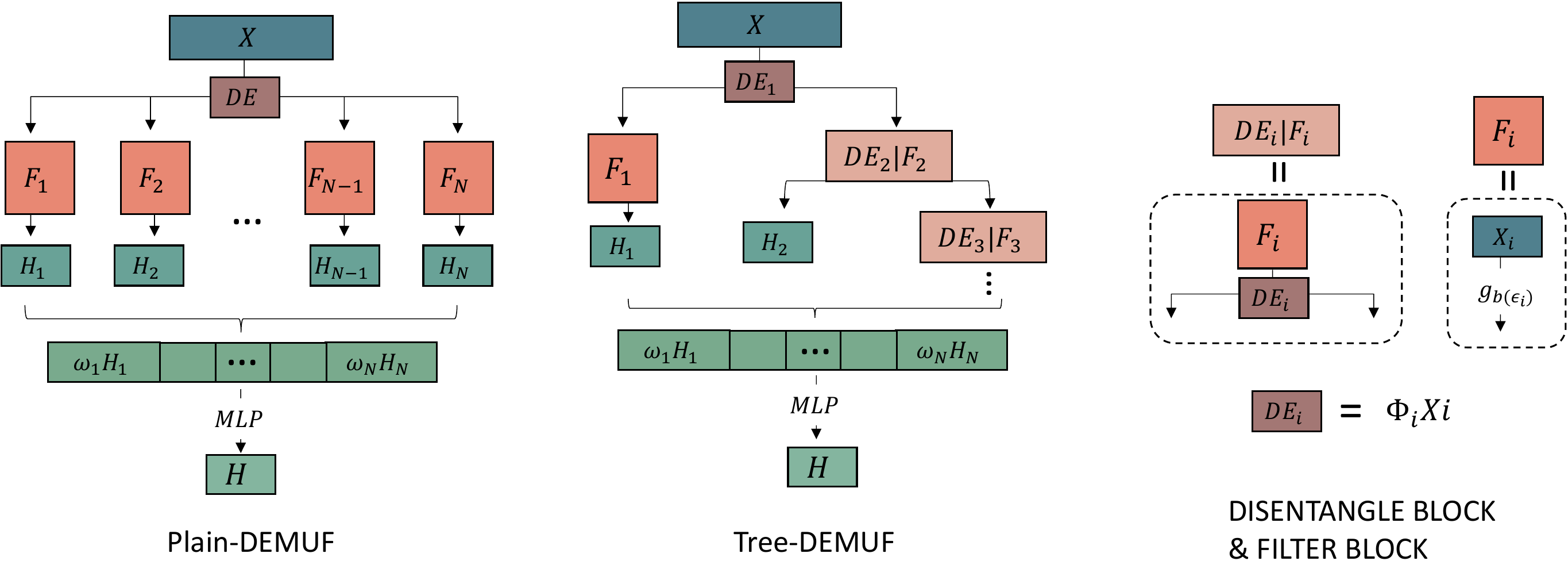}
\caption{Illustration of \textrm{Plain-DEMUF} and \textrm{Tree-DEMUF}. There are two main model blocks of \textrm{DEMUF} frameworks: \textit{disentangle block} and \textit{filter block}. In \textrm{Plain-DEMUF}, all filter blocks run in parallel as their disentangled input are generated through a single disentangle block at the same time. Differently, each \textrm{Tree-DEMUF} layer contains two branches - one is early stopped while the other will be disentangled into two branches of the next layer after going through a filter.} 
\label{Fig.model}
\vskip -0.2in
\end{figure*}

\subsection{Architecture of Two Frameworks of \textrm{DEMUF}} 
Following the conclusions in Sect.\ref{subsec.filter design}, i.e., different features should be assigned to different filters, we assemble \textbf{feature disentanglement} and \textbf{frequency filtering} blocks to our framework. The block of feature disentanglement is to divide features into different families in a learnable way. Then in the frequency filtering block, learnable graph filters are applied to targeted families of features. We provide two frameworks with different structures: \textrm{Plain-DEMUF} and \textrm{Tree-DEMUF} (depicted in Fig.~\ref{Fig.model}). 

The \textsc{Disentangle} block and \textsc{Filter} block are formulated as follows:
\begin{align*}\label{Block}
&X_k = \textrm{DISENTANGLE}(X,\Phi_k)=\Phi_k(X),\\
&H_k = \textrm{FILTER}\Big(X_k,\epsilon_k,h_k\Big)=(g_{\epsilon_k}(\tilde L))^{h_k}X_k.
\end{align*}
In our implementation, we provide two samples of \textsc{Disentangle} functions $\Phi_k$: one is linear transformations, the other is \textsc{gumbel\_softmax}~\citep{jang2016categorical} used to generate learnable masks for feature selection. In terms of the \textsc{Filter} block, we use the normalized second order filter $g_\epsilon(\tilde L)=I-\frac {((1-\epsilon)I-\tilde L)^2}{(1+|\epsilon|)^2}$ with learnable parameter $\epsilon\in(-1,1)$. In each \textsc{Filter} block, $h$ is the number of layers. 
The framework of \textrm{Plain-DEMUF} with $N$ filters is:
\begin{align*}
&H_{k}=\textrm{FILTER}\Big(\textrm{DISENTANGLE}\Big(X,\Phi_k\Big),\epsilon_k,h_k\Big),\\
&H = \textrm{MLP}\Big(\textrm{CONCAT}\Big(\Big\{H_k,\omega_k\Big|k=1,\dots,N\Big\}\Big)\Big).
\end{align*}
Based on this, we implement a simple model called \textrm{P-DEMUF}. 
Precisely, we leverage a \textsc{gumbel\_softmax} to generate $N$ learnable masks $\{M_1,\dots,M_{N}\}$ for feature sampling at once followed by different \textrm{MLP}. That is, $\Phi_k(X)=\text{MLP}_k(X\odot M_k)$. 

Similarly, we develop a model, \textrm{T-DEMUF}, under the framework of \textrm{Tree-DEMUF} formulated by:
\begin{align*}
&X_1 =\textrm{FILTER}\Big(\textrm{DISENTANGLE}\Big(X,\Psi_1\Big),\epsilon_1,h_1\Big),\\
&H_1= \textrm{FILTER}\Big(\textrm{DISENTANGLE}\Big(X,\Phi_1\Big),\epsilon,h\Big),\\
&H_{k+1}=\textrm{DISENTANGLE}\Big(X_k,\Phi_k\Big),\\
&X_{k+1} =\textrm{FILTER}\Big(\textrm{DISENTANGLE}\Big(X_k,\Psi_k\Big),\epsilon_k,h_k\Big),\\
&H = \textrm{MLP}\Big(\textrm{CONCAT}\Big(\Big\{\omega_kH_k,k=1,\dots,N\Big\}\Big)\Big).
\end{align*}
In each \textrm{T-DEMUF} layer, we use \textsc{gumbel\_softmax} with different parameters to generate two masks $M_k$ and $M'_k$ and $\Phi_k(X_k)=X_k\odot M_k$ and $\Psi_k(X_k)=X_k\odot M'_k$. In each layer, we stop further disentangling of the branch of $H_k$ by utilizing an additional constraint 
\[\mathcal{L}(X_{k-1},H_k) = \parallel X_{k-1}\odot M'_k-H_k\parallel_2^2.\] Noting that $H_k=(g_{\epsilon_k}(\tilde
L))^{h_k}X_{k-1}\odot M'_k$, this constraint is to make $(g_{\epsilon_k}(\tilde L))^{h_k}$ has high response efficiency on $H_k$.

\textbf{Model discussion.} 
Compared with filter-bank learning methods which directly apply an array of filters to features, our models use subsets of features. It can greatly reduce the amount of computation and parameters and help learning filters more effectively. \textrm{T-DEMUF} uses an additional constraint to guide the filter learning process while filters in \textrm{P-DEMUF} do not interfere with each other. We provide further discussion and ablation study in Appendix \ref{appendix}. 

\subsection{Experiments}
To validate \textrm{DEMUF}, we compare the performances of \textrm{P-DEMUF} and \textrm{T-DEMUF} with that of spectral \textrm{GNNs}, spatial \textrm{GNNs} and \textrm{MLP} on extensive datasets.

\subsubsection{Experiment Settings}

\textbf{Datasets.} We use various types of real datasets including two graphs of \textit{Citation network}~\citep{sen2008collective} - \textrm{Cora} and \textrm{Citeseer}; three subgraphs of \textit{WebKB}~\citep{pei2020geom} - \textrm{Cornell}, \textrm{Texas}, and \textrm{Wisconsin}; two \textit{Wikipedia network} - \textrm{Chameleon} and \textrm{Squirrel}~\citep{rozemberczki2021multi}; two \textit{relabeled Wikipedia network} proposed by \citet{bo2021beyond} - \textrm{Chameleon2} and \textrm{Squirrel2}; and one graph of \textit{Actor co-occurrence network} - \textrm{Actor}~\citep{tang2009social}, to validate our proposed models. More statistics of datasets and the experimental setup can be found in Appendix \ref{appendix}.

\textbf{Baselines.} We compare our models with four spectral \textrm{GNNs}: \textrm{GCN}~\citep{kipf2016semi}, \textrm{ChebNet}~\citep{defferrard2016convolutional}, \textrm{GIN}~\citep{xu2018powerful}, \textrm{ARMA}~\citep{bianchi2021graph}. We list their spectral filter forms in Appendix \ref{appendix}. 
In addition, we also add four spatial \textrm{GNNs}: \textrm{GAT}~\citep{velivckovic2017graph}, \textrm{FAGCN}~\citep{bo2021beyond}, \textrm{Geom\_GCN}~\citep{pei2020geom} and \textrm{GPRGNN}~\citep{chien2020adaptive}. Both \textrm{GAT} and \textrm{FAGCN} utilize attention mechanism, \textrm{Geom\_GCN} is a novel aggregation method based on the geometry of graph, \textrm{GPRGNN} can handle both homophilic and heterophilic graphs through learning GPR weights. Finally, we also compare with \textrm{MLP} which is an all-pass filter.

\begin{table*}[ht]
\small
\addtolength{\tabcolsep}{-2.5pt}
\caption{\textbf{Node classification accuracy.} The first row is the homophily degree. }
\vskip 0.15in
\begin{center}
\begin{tabular}{llcccccccccc}
\hline
      &           & Cora & Cite. & Cornell        & Texas          & Wisc.      & Cham.      & Squi.       & Cham.2 & Squi.2  & Actor           \\
      & $\mathcal H_1(\tilde\Pi|\mathcal G_n)$ &   0.637 &    0.586   &   0.103  &0.072 & 0.018    & -0.116  &   -0.279   &   0.127  &   0.044  &   -0.215
      \\
      \hline
\multirow{4}{*}{\rotatebox{90}{Spectral}}   & GCN       & \textrm{88.50}                     & 76.20                      & 69.02          & 66.07          & 58.50          & 64.11          & 47.60          & 75.49          & 70.87          & 32.89          \\
                          & Cheby     & 88.21                     & \textrm{76.26}                     & 81.97          & 82.79          & 82.50          & 63.61          & 50.93          & 77.68          & 69.13          & \textrm{37.23}          \\
                          & GIN       & 87.06                     & 74.10                      & 56.89          & 69.84          & 51.25          & 37.42          & 23.73          & 58.10          & 49.24          & 29.64          \\
                          & ARMA      & 87.56                     & 74.86                     & \textrm{85.25}          & \textrm{85.25}          & \textrm{92.63}          & \textbf{69.34} & \textbf{51.60} & \textbf{78.93} & \textbf{73.95} & 35.49          \\

\hline
\multirow{4}{*}{\rotatebox{90}{Spatial}}   & GAT       & 88.32                     & 76.85                     & 60.82          & 72.30           & 61.63          & 66.17          & 44.75          & 77.18          & 69.67          & 36.40           \\
                          & FAGCN     & \textbf{89.19}            & 77.15                     & 73.51          & 65.41          & 76.86          & 61.70           & 39.70           & 76.10           & 66.70           & 34.61          \\
                          & Gemo\_GCN & 85.27                     & \textbf{77.90}             & 60.81          & 67.57          & 64.12          & 60.90           & 38.14          & 73.20          & 63.30          & 31.63          \\
                          & GPRGNN    & /                         & /                         & \textbf{91.36} & \textbf{92.92} & /              & 67.48          & 49.93          & /              & /              & \textbf{39.30} \\

\hline   & \textrm{MLP}
                             & 75.33                     & 71.40                      & 92.46          & 92.46          & \textbf{95.00} & 49.56          & 34.89          & 77.28          & 63.19          & 37.38          \\
                            
\hline
\multirow{3}{*}{\rotatebox{90}{Ours}}       & T-DEMUF   & 86.72                     & 74.57                     & \textbf{92.97} & \textbf{92.79} & 93.21          & \textbf{72.03} & \textbf{59.09} & \textbf{83.31} & \textbf{74.92} & \textbf{41.11} \\
                          & P-DEMUF   & \textbf{87.85}            & \textbf{75.69}            & 91.60          & 92.04          & \textbf{94.38} & 71.47          & 57.58          & 82.40          & 74.54          & 39.18          \\
                          &v.s. Spectral& $\downarrow$ 0.65& $\downarrow$ 0.57& $\uparrow$  7.72         & $\uparrow$  7.54        & $\uparrow$ 1.75        & $\uparrow$  2.69         &$\uparrow$  7.49          & $\uparrow$  4.38         & $\uparrow$   0.97        & $\uparrow$ 3.88\\
                          &v.s. all          & $\downarrow$1.34 & $\downarrow$2.21 & $\uparrow$0.51           & $\downarrow$0.13          & $\downarrow$0.625         & $\uparrow$2.69           &$\uparrow$ 7.49           & $\uparrow$4.38           & $\uparrow$0.97           & $\uparrow$1.81 \\
\hline
\label{result}
\end{tabular}
\end{center}
\vskip -0.1in
\end{table*}

\begin{figure*}[ht]
\label{vis_result}
\centering
\vskip 0.2in 
\includegraphics[width=0.8\textwidth]{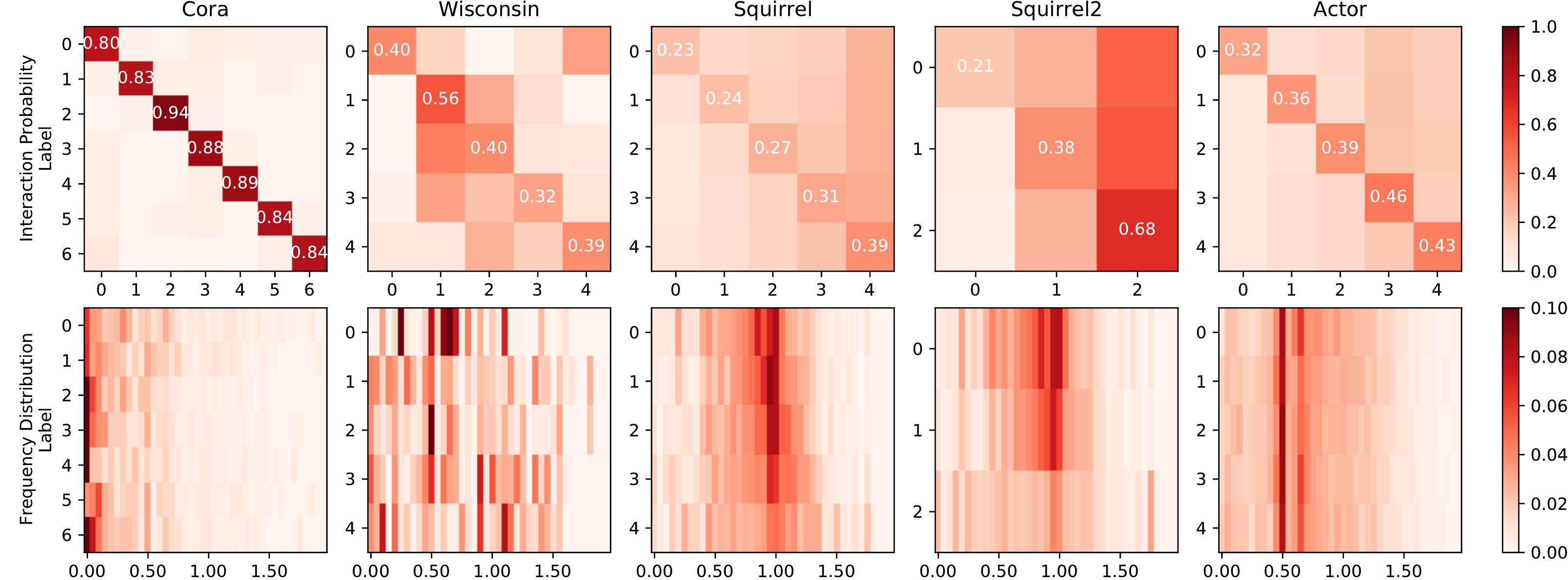}
\caption{Visualizations of interaction probability matrix and label frequency distribution of five datasets.} 
\label{Fig.compact}
\vskip -0.2in
\end{figure*}

\subsection{Result and Analysis}
We summarize the experimental results in Table \ref{result} and visualize the interaction probability and label frequency distribution of some typical graphs in Fig.~\ref{Fig.compact}. 

Our models show consistent superiority on most benchmarks and outperform all spectral \textrm{GNN} baselines. These promising results strongly suggest that our strategy of filter design proposed in Sect.\ref{subsec.filter design} is effective to improve spectral \textrm{GNNs}' prediction performance. 
Precisely, \textrm{T-DEMUF} yields over $7.49\%$ higher accuracy than the best baselines (\textrm{ARMA}) on \textrm{Squirrel}. Compared with baselines with first-order filters (\textrm{GCN} and \textrm{GIN}), our models and other baselines with high-order filters (\textrm{Cheby} and \textrm{ARMA}) have a huge lead on most of datasets apart from \textrm{Cora} and \textrm{Citeseer}. It empirically validates our argument about the superiority of high-order filters proposed in Theorem \ref{thm.first/second fitler}. On the other hand, the poor performance of \textrm{GCN} and \textrm{GIN} we implemented which are low-pass filters on graphs with low homophily degree verify our analysis on low-pass filters provided by Theorem \ref{thm.low/high fitler}. In addition, the visualization of \textrm{Cora}'s indicators confirms our deduction in Sect.\ref{Sec.Frequency distribution} - high interaction probability brings high homophily degree and high concentration of label frequency distribution on low-frequency.

Interestingly, all models have poor performance on \textrm{Actor} (less than $42\%$) while partial models including \textrm{T-DEMUF}, \textrm{P-DEMUF} obtain over $90\%$ prediction accuracy on \textrm{WebKB} networks and all have good results on \textrm{Cora}. We indicate that these observations strongly support our theoretical analysis on prediction error in terms of graph indicators in Sect.\ref{subsec. analysis response efficiency} and \ref{subsec.analysis homophily}. For example: 1. high homophily degree makes \textrm{Cora} easy to be classified; 2. highly diverse label frequency distributions which imply high label difference information content give a guarantee of good performance of appropriate filters on \textrm{Wisconsin}; 3. low homophily degree and similar label frequency distributions of \textrm{Actor} make all models fail. 


\section{Conclusion}
In this paper, we conduct a theoretical analysis on the prediction error of spectral \textrm{GNNs} and develop a deep analysis of graph filters' performance based on the introduction of significant graph indicators. We also propose an effective and practical strategy for filter design which has been empirically validated by a simple framework we developed.


\bibliography{main}
\bibliographystyle{icml2022}

\newpage
\appendix
\onecolumn
\section{Benchmarks and Model Discussion}\label{appendix}
\subsection{Statistics information of benchmarks.}
We use four types of real datasets - \textit{Citation network}, \textit{WebKB}, \textit{Actor co-occurrence network} and \textit{Wikipedia network}, to validate our proposed models. \textrm{Cora} and \textrm{Citeseer}~\citep{sen2008collective} are widely used citation benchmarks which represent paper as nodes and citation between two papers as edges. \textrm{Cornell}, \textrm{Texas}, and \textrm{Wisconsin}~\citep{pei2020geom} are three subgraphs of \textrm{WebKB} which is a webpage network with web pages as nodes and hyperlinks between them as edges. \textrm{Chameleon} and \textrm{Squirrel}~\citep{rozemberczki2021multi} are two \textrm{Wikipedia} networks with web pages as nodes and links between pages as edges. The nodes originally have five classes while \citet{bo2021beyond} proposed a new classification criteria which divides nodes into three main categories. In this paper, the relabeled networks are called \textrm{Chameleon2} and \textrm{Squirrel2}. \textrm{Actor}~\citep{tang2009social} is a subgraph of the fillm-director-actor-writer network whose nodes only represent actors and edges represent their collaborations.

We provide statistics information of our benchmarks in Table. \ref{Table.benchmark}.
\begin{table}[H]
\small
\addtolength{\tabcolsep}{-1.9pt}
\label{Table.benchmark}
\caption{Datasets statistics.}
\vskip 0.15in
\begin{center}
\begin{tabular}{lcccccccccc}
\hline
   Dataset        & Cora & Cite. & Cornell        & Texas          & Wisc.      & Cham.      & Squi.       & Cham.2 & Squi.2  & Actor           \\
 \# Nodes     & 2708 & 3327    & 183            & 183            & 251            & 2277           & 5201           & 2277           & 5201           & 7600           \\
                             \# Edges     & 5429 & 4732     & 295            & 309            & 499            & 36101          & 217073         & 36101          & 217073         & 33544          \\
                             \# Features  & 1433 & 3703     & 1703           & 1703           & 1703           & 2325           & 2089           & 2325           & 2089           & 931            \\
                            \# Classes   & 7    & 6        & 5              & 5              & 5              & 5              & 5              & 3              & 3              & 5              \\
\hline
\end{tabular}
\end{center}
\vskip -0.1in
\end{table}

\subsection{Spectral filters.} In our paper, we use four spectral \textrm{GNNs} as baselines whose spectral filters are listed as Table.\ref{Table.filter} and define a normalized second order filter $g_\epsilon(\tilde L)$ with $\epsilon\in[-1,1]$.

\begin{table}[H]
\centering
\label{Table.filter}
\caption{Spectral filters.}
\vskip 0.15in
\begin{center}
\begin{tabular}{ll}
\hline
 Model & Filter \\
 \textrm{GCN}& $I-\tilde{L}$ \\
 \textrm{GIN}& $(2+\epsilon)I-\tilde{L}$ \\
 \multirow{2}{*}{\textrm{ChebNet}}& $C^{(s)}=2C^{(2)}C^{(s-1)}-C^{(s-2)}$;\\ & $C^{(2)}=2L/\lambda_{max}-I$; $C_1=I$  \\
 \textrm{ARMA}& $(I+\sum_{k=1}^Kq^kL^k)^{-1}(\sum_{k=0}^{K-1}p_kL^k)$\\
 \textrm{Ours}& $I-\frac {((1-\epsilon)I-\tilde L)^2}{(1+|\epsilon|)^2}$\\
 \hline
\end{tabular}
\end{center}
\vskip -0.1in
\end{table}

 Noting that $g_\epsilon(\tilde L)=I-\frac {((1-\epsilon)I-\tilde L)^2}{(1+|\epsilon|)^2}=\frac {((2+|\epsilon|-\epsilon)I-\tilde L)((|\epsilon|+\epsilon)I+\tilde L)}{(1+|\epsilon|)^{2}}$, it is exactly an overlap between a low-pass filter $(2+|\epsilon|-\epsilon)I-\tilde L$ and a high-pass filter $(|\epsilon|+\epsilon)I+\tilde L$.
 
 For a given graph $\mathcal G_n$ with interaction probability $\tilde \Pi$, we investigate $\mathcal H_1(\tilde g(I-\tilde \Pi)|\mathcal G_n)$ for our second-order filter and the filter of \textrm{GCN}. We first normalize \textrm{GCN}'s filter as $g_{gcn}(\tilde L)=I-\frac {\tilde L}2$, then we have
 \begin{align*}
     &\mathcal H_1(\tilde g_{gcn}(I-\tilde \Pi)|\mathcal G_n)= \frac {(R_0-R_1)^2}{2n}+\frac 12\mathcal H_1(\tilde \Pi|\mathcal G_n)\\
     &\mathcal H_1(\tilde g_{\epsilon}(I-\tilde \Pi)|\mathcal G_n)=\frac {1+2|\epsilon|}{(1+|\epsilon|)^2}\frac {(R_0-R_1)^2}{n}+\frac {2\epsilon \mathcal H_1(\tilde \Pi|\mathcal G_n)-\mathcal H_2(\tilde \Pi|\mathcal G_n)}{(1+|\epsilon|)^2}
 \end{align*}

\subsection{Model Discussion.}


\textbf{How can our filter bank selection be data-driven.}

As we clarified in Sect.\ref{Sec.model}, our implementation of disentanglement is not random masking but learnable masking leveraging \textrm{GUMBEL-SOFTMAX}. These learnable maskings disentangle node features into several subsets of features. In our algorithm, although the form of our second order filter  $g_\epsilon(\tilde L)$ is predefined, its parameters including $\epsilon$ and weight $\omega$ are learned from specified graphs. Moreover, the parameters of our feature disentanglement blocks (the linear transformations and learnable masking) are also learned from data which will affect the learning of the filter bank. Therefore, our filter bank selection is data-driven.
 
\subsection{Experimental Setup and Additional Results}

\textbf{Experimental Setup.}

For all data, we use $60\%$ nodes for training, $20\%$ for validation and $20\%$ for testing. For all experiments, we report the mean prediction accuracy on the testing data for 10 runs. We search learning rate, hidden unit, weight decay and dropout for all models in the same search space. Finally, we choose learning rate of $0.01$, dropout rate of $0.5$, and hidden unit of $64$ over all datasets. The number of filters are searched between 2 to 10, and the final setting is: both \textrm{T-DEMUF} and  \textrm{P-DEMUF} use 2 filters with 1 layer and 3-layer MLP on \textrm{WebKB}; 2 filters with 18 layers and 3-layer MLP on \textrm{Wikipedia} and 5 filters with 1 layer and 4-layer MLP on \textrm{Actor}. For \textrm{Citation} networks, \textrm{T-DEMUF} uses 4 filters with 7 layers while \textrm{P-DEMUF} uses 3 filters with 8 layers. In addition, as the setting of benchmarks are the same as that in \textrm{Geom\_GCN}, we refer to the results reported in \cite{pei2020geom}.

\textbf{Ablation Study.}

To show the advantage of using disentanglement, we provide an ablation study on five benchmarks. Here, we propose two ablation models based on \textrm{P-DEMUF}. Recall that the disentanglement block of \textrm{P-DEMUF} consists of learnable masking and linear transformations, we design our ablation models by taking off the component of masking and linear transformation. Also, for fair and intuitive comparison, we simply fix the number of filters as 2. The results shown as Table.\ref{Table.ablation} validate that if we take off the disentanglement blocks of \textrm{P-DEMUF}, the results become worse in most of benchmarks.

\begin{table}[H]
\centering
\label{Table.ablation}
\caption{Node classification accuracy of P-DEMUF and its ablation models. We fix the number of second order filters as 2.}
\vskip 0.1in
\begin{center}
\begin{tabular}{llllll}
\hline
Dataset                         & Cornell  & Texas    & Cham.    & Squi.    & Actor    \\
P-DEMUF                         & 81.08 & 85.14 & 68.46 & 55.45 & 36.95 \\
P-DEMUF (w/o masking )          & 80.00 & 84.86 & 67.20 & 53.54 & 37.20 \\
P-DEMUF (w/o masking \& linear) & 77.30 & 86.49 & 62.5  & 44.19 & 36.71 \\
 \hline
\end{tabular}
\end{center}
\vskip -0.1in
\end{table}

\section{Proof of Theorem}\label{append.proof}

\subsection{Proof of Theorem \ref{thm.lower bound}}\label{AP.proof lower bound}
Before presenting the proof of Theorem \ref{thm.lower bound}, we start with a useful lemma.
\begin{lemma}\label{lm.lower bound}
Denote $x\in \mathbb R$, $y\in\{0,1\}$, $\mathcal S(x,y)=\{x<-1,y=1\; \text{or}\;x>1,y=0\}$ and $\psi$ is the clamp function defined as
$$
\psi(x)=\min\{\max\{x,-1\},1\}=  \left\{
\begin{array}{ll}
      1 & x>1 \\
      x & -1<x<1 \\
      -1 & x<-1 \\
\end{array} 
\right.,
$$
then we have
$$
(\frac 1{1+e^x}-y)^2\ge \frac 14-\frac {(1-2y)\psi(x)}4+\frac {\psi(x)^2}{16}-\frac {|\psi(x)|^3}{48}-\frac {|\psi(x)|^4}{96}-\frac {1}{(1+e)^2}\mathbf 1_{\mathcal S}(x,y),
$$
where, $\mathbf 1_{\mathcal S}(x,y)$ is the indication function of ${\mathcal S(x,y)}$.
\end{lemma}

\begin{proof}

Noting that
\[
error(x,\psi(x))= (\frac {1}{1+e^x}-y)^2-(\frac {1}{1+e^{\psi(x)}}-y)^2\in \left\{
\begin{array}{ll}
      [-(\frac 1{1+e})^2,0] & x>1,\;y=0 \\
      
      [0,(\frac 1{1+e})^2] & x>1,\;y=1 \\
      
      [0,(\frac 1{1+e})^2] & x<-1,\;y=0 \\
      
      [-(\frac 1{1+e})^2,0] & x<-1,\;y=1 \\
\end{array} 
\right.
\]
then we have 
$$
error(x,\psi(x))\ge -\frac 1{(1+e)^2}\mathbf 1_{\mathcal S}(x,y).
$$
For $x\in[-1,1]$, the first-order (also the second order) Taylor expansion of $\frac 1{1+e^x}$ is $\frac 12 -\frac 14x$. 

Denote $R(x)$ as the remainder term, i.e. $R(x)=\frac 1{1+e^x}-\frac 12 +\frac 14x$, since 
$$
(\frac 1{(1 + e^x)^2})''' = -\frac {e^x (-4 e^x + e^{2x} + 1)}{(1 + e^x)^4}\le (\frac 1{(1 + e^x)^2})'''\Big|_{x=0}=\frac 18,
$$
we have 
$$
|R(x)|\le \max\Big|(\frac 1{(1 + e^x)^2})'''\Big|\frac {|x|^{3}}{3!}=\frac {|x|^3}{48}.
$$
Therefore, 
\begin{align*}
&(\frac 1{1+e^x}-y)^2=(\frac 1{1+e^{\psi(x)}}-y)^2+error(x,\psi(x))\\
\ge& (\frac 1{1+e^{\psi(x)}}-y)^2-\frac {1}{(1+e)^2}\mathbf 1_{\mathcal S}(x,y)\\
=&(\frac 12 -\frac 14\psi(x)-y+R(\psi(x)))^2-\frac {1}{(1+e)^2}\mathbf 1_{\mathcal S}(x,y)\\
\ge &(\frac 12 -\frac 14\psi(x)-y)^2-2|R(\psi(x))||\frac 12 -\frac 14\psi(x)-y|-\frac {1}{(1+e)^2}\mathbf 1_{\mathcal S}(x,y)\\
\ge& (\frac 12 -\frac 14\psi(x)-y)^2-\frac {|\psi(x)|^3}{24}\big|\frac 12 -\frac 14\psi(x)-y\big|-\frac {1}{(1+e)^2}\mathbf 1_{\mathcal S}(x,y)\\
\ge& (\frac 12 -\frac 14\psi(x)-y)^2-\frac {|\psi(x)|^3}{48}-\frac {|\psi(x)|^4}{96}-\frac {1}{(1+e)^2}\mathbf 1_{\mathcal S}(x,y)\\
=& \frac 14-\frac {(1-2y)\psi(x)}4+\frac {\psi(x)^2}{16}-\frac {|\psi(x)|^3}{48}-\frac {|\psi(x)|^4}{96}-\frac {1}{(1+e)^2}\mathbf 1_{\mathcal S}(x,y)
\end{align*}
\end{proof}

Below, we provide the proof of Theorem \ref{thm.lower bound}.

\textbf{Proof of Theorem \ref{thm.lower bound}}
\begin{proof}
According to Lemma \ref{lm.lower bound},
\begin{align*}
&Er(\mathbf x_0,\mathbf y_0)=\parallel \frac {e^{g(\tilde L)\mathbf x_0}}{e^{g(\tilde L)\mathbf x_0}+e^{g(\tilde L)\mathbf x_1}}-\mathbf y_0\parallel_2^2=\sum_l(\frac 1{1+e^{g(\tilde L)(\mathbf x_{1l}-\mathbf x_{0l})}}-\mathbf y_{0l})^2\\
\ge& \sum_l\frac 14-\frac {(1-2\mathbf y_{0l})\psi(\mathbf z_l)}4+\frac {\psi(\mathbf z_l)^2}{16}-\frac {|\psi(\mathbf z_l)|^3}{48}-\frac {|\psi(\mathbf z_l)|^4}{96}-\frac {I_{\mathcal S}(\mathbf z_l,y)}{(1+e)^2}
\\
=& \frac n4+\frac {\parallel \psi(\mathbf z)\parallel_2^2}{16}-\frac 12(\frac 12-\mathbf y_0)^\top \psi(\mathbf z)-\frac C{(1+e)^2}-\sum_l\frac {|\psi(\mathbf z_l)|^3}{24}\big|\frac 12 -\frac 14\psi(\mathbf z_l)-\mathbf y_{0l}\big|\\
\ge& \frac n4-\frac 14(\mathbf y_1-\mathbf y_0)^\top \psi(\mathbf z)+\frac {\parallel \psi(\mathbf z)\parallel_2^2}{16}-\frac {\parallel \psi(\mathbf z)\parallel_3^3}{48}-\frac {\parallel \psi(\mathbf z)\parallel_4^4}{96}-\frac C{(1+e)^2}\\
=&\tilde {Er}(\mathbf x_0,\mathbf y_0).
\end{align*}

Noting that $C\le \parallel \psi(\mathbf z)\parallel_4^4\le \parallel \psi(\mathbf z)\parallel_3^3\le \parallel \psi(\mathbf z)\parallel_2^2\le n$, then we have
\begin{align*}
\tilde {Er}(\mathbf x_0,\mathbf y_0)&\ge (\frac n4-\frac {\parallel \psi(\mathbf z)\parallel_3^3}{32}+(\frac 1{16}-\frac 1{(1+e)^2})C-\frac 14(\mathbf y_1-\mathbf y_0)^\top \psi(\mathbf z)\\
&>\frac n4-\frac {\parallel \psi(\mathbf z)\parallel_3^3}{32}-\frac C{100}-\frac 14(\mathbf y_1-\mathbf y_0)^\top \psi(\mathbf z)\\
&>\frac n4-\frac {33}{800}\parallel \psi(\mathbf z)\parallel_3^3-\frac 14(\mathbf y_1-\mathbf y_0)^\top \psi(\mathbf z)\\
&>\frac {167}{800}n-\frac 14(\mathbf y_1-\mathbf y_0)^\top \psi(\mathbf z)\\
&=\frac {167}{800}n-\frac 14(\mathbf y_1-\mathbf y_0)^\top \psi(g(\tilde L)(\mathbf x_{1}-\mathbf x_{0}))\\
&=\frac {167}{800}n-\frac 14 \sum_l\psi\big((\mathbf y_{1l}-\mathbf y_{0l})(g(\tilde L)(\mathbf x_{1}-\mathbf x_{0})_l\big).
\end{align*}
\end{proof}

\subsection{Proof of Proposition \ref{pro.interaction inequality} and \ref{The.moment}} \label{AP.proof interaction}
We first introduce two useful lemma.

\begin{lemma}\label{lm.moment}
For $\mathcal{G}=\{\mathcal{V},\mathcal{E}\}$, let $\mathbf{f}$ be the frequency of signal $\mathbf{x}$, then $\mathbb{E}[\mathbf{f}^n]=\frac{\textbf{x}^\top(I-\tilde{A})^n\textbf{x}}{\textbf{x}^\top\textbf{x}}$.
\end{lemma}

\begin{proof}
Since $\mathbf{x}=\sum\limits_{i=0}^{n-1}\alpha_i \mathbf{u}_i$, $\mathbf{u}_i$ is the $i$-th unit eigenvector of $\tilde{L}$ and $\lambda^n=\mathbf{u}_i^\top\tilde{L}^n\mathbf{u}_i$ then we have
\begin{align*}
     \mathbb{E}[\mathbf{f}^n]=\sum\limits_{i=0}^{n-1}P(\mathbf{f}=\lambda_i)\lambda_i^n=\frac {\sum(\alpha_i\mathbf{u}_i)^\top\tilde{L}^n(\alpha_i\mathbf{u}_i)}{\sum\alpha_i^2}=\frac{{\mathbf{x}^\top}\tilde{L}^n\mathbf{x}}{\mathbf{x}^\top \mathbf{x}}=\frac{{\mathbf{x}^\top}(I-\tilde{A})^n\mathbf{x}}{\mathbf{x}^\top \mathbf{x}}.
\end{align*}
\end{proof}

\begin{lemma}\label{lemma}
Let $B\in \mathbb{R}^{n\times n}$ is a symmetric matrix, $\forall ij$, $\mathbf{y}\in \mathbb{R}^n$, we have 
\[\frac{\mathbf{y}^\top B^2\mathbf{y}}{\mathbf{y}^\top \mathbf{y}}\ge \Big(\frac{\mathbf{y}^\top B\mathbf{y}}{\mathbf{y}^\top \mathbf{y}}\Big)^2.\]
\end{lemma}
\begin{proof}
Since $B$ is symmetric, then we have $B=U\Lambda U^\top$, here $U$ is matrix of unit eigenvectors of $B$. From the proof of Lemma \ref{lm.moment}, we obtain that $\frac{\mathbf{y}^\top B^2\mathbf{y}}{\mathbf{y}^\top \mathbf{y}}=\frac {\sum(\alpha_i\lambda_i)^2}{\sum\alpha_i^2}$ and $\Big(\frac{\mathbf{y}^\top B\mathbf{y}}{\mathbf{y}^\top \mathbf{y}}\Big)^2=\frac{(\sum\alpha_i^2\lambda_i)^2}{(\sum\alpha_i^2)^2}$.

From Hölder's inequality, we have $(\sum(\alpha_i\lambda_i)^2)(\sum\alpha_i^2)\ge (\sum\alpha_i^2\lambda_i)^2$.

Therefore, we have $\frac {\sum(\alpha_i\lambda_i)^2}{\sum\alpha_i^2}\ge \frac{(\sum\alpha_i^2\lambda_i)^2}{(\sum\alpha_i^2)^2}$.
\end{proof}

Below is the proof of Proposition \ref{pro.interaction inequality}.

\textbf{Proof of proposition \ref{pro.interaction inequality}.}
\begin{proof}
For $P=D^{-1}A$ and $\tilde{A}=D^
{-\frac 12}AD^{-\frac 12}$, and $\Pi$, $\tilde{\Pi}$ defined by Definition~\ref{def.interaction},
\begin{align*}
    &R_l\Pi_{lm}^k+R_m\Pi_{ml}^k\ge \sqrt{R_lR_m}\tilde \Pi_{lm}^k\\
    \Longleftrightarrow& (R\Pi^k+(\Pi^k)^\top R)
_{lm}\ge 2(R^{\frac 12}\tilde{\Pi}^{k}R^{\frac 12})_{lm}\\
\Longleftrightarrow&\mathbf{y}_{m}^\top(P^k+(P^k)^\top)\mathbf{y}_l\ge 2\mathbf{y}_{m}^\top \tilde{A}^k\mathbf{y}_l\\
\Longleftrightarrow& \mathbf{y}_{m}^\top(P^k+DP^kD^{-1})\mathbf{y}_l\ge 2\mathbf{y}_{m}^\top D^{\frac 12}P^kD^{-\frac 12}\mathbf{y}_l\\
\Longleftrightarrow& \mathbf{y}_{m}^\top(P^k+DP^kD^{-1}-2D^{\frac 12}P^kD^{-\frac 12})\mathbf{y}_l\ge 0
\end{align*}
Since $(P^k+(P^k)^\top)_{ij}=P^k_{ij}+\frac{d_i}{d_j}P^k_{ij}\ge 2\sqrt{\frac{d_i}{d_j}}P^k_{ij}=2\tilde{A}^k_{ij}$, we have $\mathbf{y}_{m}^\top(P^k+(P^k)^\top-2\tilde{A}^k)\mathbf{y}_{l}\ge 0$. 

Let $m=l$, then we get $\pi_l^k\ge \tilde{\pi}_l^k$. 

According to Lemma \ref{lemma}, let $g(\cdot)=(\cdot)^n$, since $g(\tilde{A})$ is symmetric, then we have 
\[(g^2[\tilde{\Pi}])_{ll}=\frac{\mathbf{y}^\top (g(\tilde{A}))^{2}\mathbf{y}}{\mathbf{y}^\top \mathbf{y}}\ge \Big(\frac{\mathbf{y}^\top g(\tilde{A})\mathbf{y}}{\mathbf{y}^\top \mathbf{y}}\Big)^2=(g[\tilde{\Pi}]_{ll})^2\Rightarrow \tilde \pi_l^{2k}\ge (\tilde \pi_l^k)^2.
\]
\end{proof}

\textbf{Proof of Proposition \ref{The.moment}.}

Proposition \ref{The.moment} can be directly derived from Lemma \ref{lm.moment}.


\subsection{Proof of Theorem \ref{thm.upper bound 1} and \ref{thm.upper bound 2}} \label{AP.proof upper bound}

Denote indicator set  $\mathcal I_g=\{i|g(\lambda_i)\ne 0, i=0,\dots,n-1\}$, $\mathcal I_{\delta}=\{i|\delta_i\ne 0, i=0,\dots,n-1\}$ and $\mathcal I_{g,\delta}=\mathcal I_{g}\cap\mathcal I_{\delta}$, it is obvious that  $\sum_{i=0}^{n-1}\psi(\eta_ig(\lambda_i)\delta_i)=\sum_{i\in \mathcal I_{g,\delta,\eta}}\psi(\eta_ig(\lambda_i)\delta_i)$. For any $\eta$, we construct $\tilde \eta_i= \psi_{\frac 1{g(\lambda_i)\delta_i}}(\eta_i)$ such that $|\tilde \eta_ig(\lambda_i)\delta_i|\le 1$ and $\sum_{i\in \mathcal I_{g,\delta,\eta}}\psi(\eta_ig(\lambda_i)\delta_i)=\sum_{i\in \mathcal I_{g,,\delta,\tilde \eta}}\tilde \eta_ig(\lambda_i)\delta_i$.

\begin{lemma}\label{lm.upper bound}
For any $g(\cdot)$ and $\delta$,
\begin{align}
\sum_{i\in \mathcal I_{g,\delta}}g(\lambda_i)\le \frac {\sum_{i\in \mathcal I_{g,\delta}}\log p_i}{\log (1+c(g,\delta))\mu_g(\delta)}\le \frac {\sum_{i\in \mathcal I_{\delta}}\log p_i}{\log (1+c(g,\delta))\mu_g(\delta)},
\end{align}
here $c(g,\delta)=\frac {\sum_{i\in \mathcal I_{g}-\mathcal I_{g,\delta}}g(\lambda_i)}{\sum_{i\in \mathcal I_{g,\delta}}g(\lambda_i)}$ and  $p_i=\frac {\delta_i^2}{\sum_{k=0}^{n-1}\delta_k^2}$.
\end{lemma}

\begin{proof}
According to the weighted AM-GM inequality, we have
\begin{align*}
&\frac {\sum_{i\in \mathcal I_{g,\delta}}g(\lambda_i)p_i}{\sum_{i\in \mathcal I_{g,\delta}}g(\lambda_i)}\ge \prod_{i\in \mathcal I_{g,\delta}}p_i^{\frac {g(\lambda_i)}{\sum_{k\in \mathcal I_{g,\delta}}g(\lambda_k)}}\ge (\prod_{i\in \mathcal I_{g,\delta}}p_i)^{\frac {1}{\sum_{k\in \mathcal I_{g,\delta}}g(\lambda_k)}}\\
\Rightarrow & \sum_{i\in \mathcal I_{g,\delta}}g(\lambda_i)\le \frac {\sum_{i\in \mathcal I_{g,\delta}}\log p_i}{\log \sum_{i\in \mathcal I_{g,\delta}}g(\lambda_i)p_i-\log \sum_{i\in \mathcal I_{g,\delta}}g(\lambda_i)}
\end{align*}

Noting that $\mu_g(\delta)=\frac {\sum_{i=0}^{n-1}g(\lambda_i)p_i}{\sum_{i=0}^{n-1}g(\lambda_i)}=\frac {\sum_{i\in \mathcal I_{g,\delta}}g(\lambda_i)p_i}{\sum_{i\in \mathcal I_{g}}g(\lambda_i)}$, we have 
$$
\frac {\sum_{i\in \mathcal I_{g,\delta}}g(\lambda_i)p_i}{\sum_{i\in \mathcal I_{g,\delta}}g(\lambda_i)}=\frac {\sum_{i\in \mathcal I_{g}}g(\lambda_i)}{\sum_{i\in \mathcal I_{g,\delta}}g(\lambda_i)}\mu_g(\delta)=(1+\frac {\sum_{i\in \mathcal I_{g}-\mathcal I_{g,\delta}}g(\lambda_i)}{\sum_{i\in \mathcal I_{g,\delta}}g(\lambda_i)})\mu_g(\delta)=(1+c(g,\delta))\mu_g(\delta)< 1,
$$
here $c(g,\delta)=\frac {\sum_{i\in \mathcal I_{g}-\mathcal I_{g,\delta}}g(\lambda_i)}{\sum_{i\in \mathcal I_{g,\delta}}g(\lambda_i)}$. 

Then we obtain that 
$$
\sum_{i\in \mathcal I_{g,\delta}}g(\lambda_i)\le \frac {\sum_{i\in \mathcal I_{g,\delta}}\log p_i}{\log (1+c(g,\delta))\mu_g(\delta)}\le \frac {\sum_{i\in \mathcal I_{\delta}}\log p_i}{\log (1+c(g,\delta))\mu_g(\delta)}.
$$

\end{proof}

\begin{corollary}\label{co.lower bound}
For any $g(\cdot)$, $\delta$ and $\tilde \eta$ with  $p_i=\frac {\delta_i^2}{\sum_{k=0}^{n-1}\delta_k^2}$ and $q_i=\frac {\tilde \eta_i^2}{\sum_{k=0}^{n-1}\tilde \eta_k^2}$, we have
\begin{align}
\sum_{i\in \mathcal I_{g,\delta,\tilde \eta}}g(\lambda_i)\le \min\{\frac {\sum_{i\in \mathcal I_{g,\delta}}\log p_i}{\log (1+c(g,\delta))\mu_g(\delta)}, \frac {\sum_{i\in \mathcal I_{g,\tilde \eta}}\log q_i}{\log (1+c(g,\tilde \eta))\mu_g(\tilde \eta)}\}.
\end{align}
\end{corollary}

\begin{proof}
Since $g(\lambda_i)\in[0,1]$, $\mathcal I_{g,\delta,\tilde \eta}\subseteq \mathcal I_{g,\delta}$ and $\mathcal I_{g,\delta,\tilde \eta}\subseteq \mathcal I_{g,\tilde \eta}$, then $$\sum_{i\in \mathcal I_{g,\delta,\tilde \eta}}g(\lambda_i)\le \min\{\sum_{i\in \mathcal I_{g,\delta}}g(\lambda_i),\sum_{i\in \mathcal I_{g,\tilde \eta}}g(\lambda_i)\}.$$ 
With this observation and Lemma \ref{lm.upper bound}, we can obtain the corollary.
\end{proof}

Below is the proof of Theorem \ref{thm.upper bound 1} and \ref{thm.upper bound 2}.

\textbf{Proof of Theorem \ref{thm.upper bound 1}.}
\begin{proof}

Since  $m_g=\min_{i\in \mathcal I_{g,\delta,\tilde \eta}}g(\lambda_i)>0$,  $\max_{i\in \mathcal I_{g,\delta,\tilde \eta}}|\tilde \eta_i\delta_i|\le \frac 1{m_g}$. With Lemma \ref{lm.lower bound} and its corollary \ref{co.lower bound}, we obtain the upper bound of $\sum_{i=0}^{n-1}\psi(\eta_ig(\lambda_i)\delta_i)$:
\begin{align*}
&\sum_{i=0}^{n-1}\psi(\eta_ig(\lambda_i)\delta_i)=\sum_{i\in \mathcal I_{g,\delta,\tilde \eta}}\psi(\eta_ig(\lambda_i)\delta_i)=\sum_{i\in \mathcal I_{g,\delta,\tilde \eta}}\tilde \eta_ig(\lambda_i)\delta_i\\
\le& \sum_{i\in \mathcal I_{g,\delta,\tilde \eta}}g(\lambda_i)|\tilde \eta_i\delta_i|\le \frac {\sum_{i\in \mathcal I_{g,\delta,\tilde \eta}}g(\lambda_i)}{m_g}\\
\le&\frac 1{m_g}\min\{\frac {\sum_{i\in \mathcal I_{g,\delta}}\log p_i}{\log (1+c(g,\delta))\mu_g(\delta)}, \frac {\sum_{i\in \mathcal I_{g,\tilde \eta}}\log q_i}{\log (1+c(g,\tilde \eta))\mu_g(\tilde \eta)}\}\\
\le&\frac 1{m_g}\min\{\frac {\sum_{i\in \mathcal I_{\delta}}\log p_i}{\log (1+c(g,\delta))\mu_g(\delta)}, \frac {\sum_{i\in \mathcal I_{\tilde \eta}}\log q_i}{\log (1+c(g,\tilde \eta))\mu_g(\tilde \eta)}\}\\
=&\frac 1{m_g}\min\{\frac {-\mathbf I(\delta)}{\log (1+c(g,\delta))\mu_g(\delta)}, \frac {-\mathbf I(\tilde \eta)}{\log (1+c(g,\tilde \eta))\mu_g(\tilde \eta)}\}
\end{align*}
\end{proof}

\textbf{Proof of Theorem \ref{thm.upper bound 2}.}
\begin{proof}

Let $\Delta y=\mathbf y_0-\mathbf y_1$ and $\gamma,\omega$ be the spectra of $\mathbf y_0$ and $\mathbf y_1$, respectively. Then $\delta=\gamma-\omega$.

Since $\parallel\gamma-\omega\parallel_2^2=\parallel\mathbf y_0-\mathbf y_1\parallel_2^2=n$, $\mathbf I(\gamma-\omega)=\sum_i\log\frac {(\gamma_i-\omega_i)^2}{\parallel\gamma-\omega\parallel_2^2}=2\sum_i\log|\gamma_i-\omega_i|-n\log n$.

Denote $p_i$ as the frequency probability of $\Delta y$, according to the definition, we have $p_i=\frac {(\gamma_i-\omega_i)^2}{\parallel\gamma-\omega\parallel_2^2}=\frac {\gamma_i^2+\omega_i^2-2\gamma_i\omega_i}{n}$ . Recall the definition of the measure $\mu_g(\cdot)$ and symmetric interaction probability, we represent $\mu_g(\delta)$ as:
\begin{align*}
&\mu_g(\delta)=\frac {\sum_ig(\lambda_i)p_i}{\sum_ig(\lambda_i)}\\
=&\frac {\sum_ig(\lambda_i)(\gamma_i^2+\omega_i^2-2\gamma_i\omega_i)}{n\sum_ig(\lambda_i)}\\
=&\frac {\gamma^\top g(\Lambda)\gamma+\omega^\top g(\Lambda)\omega-2\gamma^\top g(\Lambda)\omega}{n\sum_ig(\lambda_i)}\\
=&\frac {\mathbf y_0^\top g(\tilde L)\mathbf y_0+\mathbf y_1^\top g(\tilde L)\mathbf y_1-2\mathbf y_0^\top g(\tilde L)\mathbf y_1}{n\sum_ig(\lambda_i)}\\
=&\frac {R_0(\tilde g(I-\tilde \Pi))_{00}+R_1\tilde g(I-\tilde \Pi))_{11}-2\sqrt {R_0R_1}\tilde g(I-\tilde \Pi))_{01}}{n\sum_ig(\lambda_i)}
\end{align*}

 As we show in the proof of Lemma \ref{lm.upper bound},
 $$(1+c(g,\delta))\mu_g(\delta)=\frac {\sum_{i\in \mathcal I_{g}}g(\lambda_i)}{\sum_{i\in \mathcal I_{g,\delta}}g(\lambda_i)}\mu_g(\delta).$$ 
 Since $\sum_{i\in \mathcal I_{g}}g(\lambda_i)=\sum_{i}g(\lambda_i)$, we have

\begin{align*}
&(1+c(g,\delta))\mu_g(\delta)\\
=&\frac {R_0(\tilde g(I-\tilde \Pi))_{00}+R_1\tilde g(I-\tilde \Pi))_{11}-2\sqrt {R_0R_1}\tilde g(I-\tilde \Pi))_{01}}{n\sum_{i\in \mathcal I_{g,\delta}}g(\lambda_i)}\\
=&\frac{\mathcal H_1(\tilde g(I-\tilde \Pi)|\mathcal G_n)}{\sum_{i\in \mathcal I_{g,\delta}}g(\lambda_i)}.
\end{align*}
Based on the Corollary \ref{co.upper bound}, we have
\begin{align*}
    Er(X,Y)=2Er(\mathbf x_0,\mathbf y_0)>\frac {167}{400}n +\frac {\mathbf I(\delta)}{2m_g\log\frac{\mathcal H_1(\tilde g(I-\tilde \Pi)|\mathcal G_n)}{\sum_ig(\lambda_i)}}.
\end{align*}
\end{proof}

\subsection{Proof of Theorem \ref{thm.low/high fitler} and \ref{thm.first/second fitler}}

\textbf{Proof of Theorem \ref{thm.low/high fitler}.}
\begin{proof}
To prove this theorem is equivalent to prove that for $g_1(\tilde L)=\epsilon_1I+\epsilon_2\tilde L\in\mathcal S_g^1$ and $g_2(\tilde L)=\epsilon'_1I+\epsilon'_2\tilde L+\epsilon'_3\tilde L^2\in\mathcal S_g^2$,  then for homophilic graphs with $\mathcal H_1(\tilde \Pi|\mathcal G_n)>0$, we can always find a low-pass filter work better than a high-pass filter; otherwise, we can always find a high-pass filter work better than a low-pass filter.

Dentoe $\int_{0}^2g_1(\lambda)d\lambda=2(\epsilon_1+\epsilon_2)=a$ and $\int_{0}^2g_2(\lambda)d\lambda=\epsilon'_1+\epsilon'_2+\frac 43 \epsilon'_3=b$, then we have

\begin{align*}
\mathcal H_1(\tilde g_1(I-\tilde \Pi)|\mathcal G_n)&=\frac {R_0}n(\tilde g_1(I-\tilde \Pi))_{00}+\frac {R_1}n(\tilde g_1(I-\tilde \Pi))_{11}-2\frac {\sqrt {R_0R_1}}n(\tilde g_1(I-\tilde \Pi))_{01}\\
&=(\epsilon_1+\epsilon_2)\frac {(\sqrt{R_0}-\sqrt{R_1})^2}{n}-\epsilon_2\frac {(R_0\tilde \pi_0+R_1\tilde \pi_1-2\sqrt{R_0R_1}\tilde \Pi_{01})}n\\
&= (\epsilon_1+\epsilon_2)(1-\frac {2\sqrt{R_0R_1}}{n})-\epsilon_2\mathcal H_1(\tilde \Pi|\mathcal G_n)\\
&= \frac a2(1-\frac {2\sqrt{R_0R_1}}{n})-\epsilon_2\mathcal H_1(\tilde \Pi|\mathcal G_n)\\
\mathcal H_1(\tilde g_2(I-\tilde \Pi)|\mathcal G_n)&=\frac {R_0}n(\tilde g_2(I-\tilde \Pi))_{00}+\frac {R_1}n(\tilde g_2(I-\tilde \Pi))_{11}-2\frac {\sqrt {R_0R_1}}n(\tilde g_2(I-\tilde \Pi))_{01}\\
&=(\epsilon'_1+\epsilon'_2+\epsilon'_3)\frac {(\sqrt{R_0}-\sqrt{R_1})^2}{n}-(\epsilon'_2+2\epsilon'_3)\frac {(R_0\tilde \pi_0+R_1\tilde \pi_1-2\sqrt{R_0R_1}\tilde \Pi_{01})}n\\
&\quad+ \epsilon'_3\frac {(R_0\tilde \pi_0^2+R_1\tilde \pi_1^2-2\sqrt{R_0R_1}\tilde \Pi_{01}^2)}n\\
&= (\epsilon'_1+\epsilon'_2+\epsilon'_3)(1-\frac {2\sqrt{R_0R_1}}{n})-(\epsilon'_2+2\epsilon'_3)\mathcal H_1(\tilde \Pi|\mathcal G_n)+\epsilon'_3\mathcal H_2(\tilde \Pi|\mathcal G_n)\\
&= (\frac b2-\frac 13\epsilon'_3)(1-\frac {2\sqrt{R_0R_1}}{n})-(\epsilon'_2+2\epsilon'_3)\mathcal H_1(\tilde \Pi|\mathcal G_n)+\epsilon'_3\mathcal H_2(\tilde \Pi|\mathcal G_n)\\
\end{align*}

Given a graph with $\mathcal H_1(\tilde \Pi|\mathcal G_n)>0$, for a given first-order high-pass filter $g(\tilde L)=\epsilon_1I+|\epsilon_2|\tilde L$, we can always find a low-pass filter such as $g'(\tilde L)=(\epsilon_1+2|\epsilon_2|)I-|\epsilon_2|\tilde L$ such that
$$
\mathcal H_1(\tilde g(I-\tilde \Pi)|\mathcal G_n)<\mathcal H_1(g'(I-\tilde \Pi)|\mathcal G_n)=\frac a2(1-\frac {2\sqrt{R_0R_1}}{n})+|\epsilon_2|\mathcal H_1(\tilde \Pi|\mathcal G_n).
$$
For a given second order high-pass filter $g(\tilde L)=\epsilon'_1I+\epsilon'_2\tilde L+\epsilon'_3\tilde L^2$ with $\epsilon'_2+2\epsilon'_3>0$, we can always find a low-pass filter such as $g'(\tilde L)=(\epsilon'_1+2\epsilon'_2+4\epsilon'_3)I-(\epsilon'_2+4\epsilon'_3)\tilde L+\epsilon'_3\tilde L^2$ such that
$$
\mathcal H_1(\tilde g(I-\tilde \Pi)|\mathcal G_n)<\mathcal H_1(g'(I-\tilde \Pi)|\mathcal G_n)= (\frac b2-\frac 13\epsilon'_3)(1-\frac {2\sqrt{R_0R_1}}{n})+(\epsilon'_2+2\epsilon'_3)\mathcal H_1(\tilde \Pi|\mathcal G_n)+\epsilon'_3\mathcal H_2(\tilde \Pi|\mathcal G_n).
$$

Similarly, Given a graph with $\mathcal H_1(\tilde \Pi|\mathcal G_n)<0$, for a given first-order low-pass filter $g(\tilde L)=\epsilon_1I-|\epsilon_2|\tilde L$, we can always find a high-pass filter such as $g'(\tilde L)=(\epsilon_1-2|\epsilon_2|)I+|\epsilon_2|\tilde L$ such that
$$
\mathcal H_1(\tilde g(I-\tilde \Pi)|\mathcal G_n)<\mathcal H_1(g'(I-\tilde \Pi)|\mathcal G_n)=\frac a2(1-\frac {2\sqrt{R_0R_1}}{n})+|\epsilon_2\mathcal H_1(\tilde \Pi|\mathcal G_n)|.
$$
For a given second order low-pass filter $g(\tilde L)=\epsilon'_1I+\epsilon'_2\tilde L+\epsilon'_3\tilde L^2$ with $\epsilon'_2+2\epsilon'_3<0$, we can always find a high-pass filter such as $g'(\tilde L)=(\epsilon'_1+2\epsilon'_2+4\epsilon'_3)I-(\epsilon'_2+4\epsilon'_3)\tilde L+\epsilon'_3\tilde L^2$ such that
\begin{align*}
\mathcal H_1(\tilde g(I-\tilde \Pi)|\mathcal G_n)<\mathcal H_1(g'(I-\tilde \Pi)|\mathcal G_n)= (\frac b2-\frac 13\epsilon'_3)(1-\frac {2\sqrt{R_0R_1}}{n})+|(\epsilon'_2+2\epsilon'_3)\mathcal H_1(\tilde \Pi|\mathcal G_n)|+\epsilon'_3\mathcal H_2(\tilde \Pi|\mathcal G_n).
\end{align*}

Based on the above conclusions, for a given graph $\mathcal G_n$, for given filter $g_1$,  $g_2$ there exists $g_1'$ and $g_2'$ such that
\begin{align*}
\mathcal H_1(\tilde g'_1(I-\tilde \Pi)|\mathcal G_n)&=\frac a2(1-\frac {2\sqrt{R_0R_1}}{n})+|\epsilon_2\mathcal H_1(\tilde \Pi|\mathcal G_n)|\\
\mathcal H_1(\tilde g'_2(I-\tilde \Pi)|\mathcal G_n)&=(\frac b2-\frac 13\epsilon'_3)(1-\frac {2\sqrt{R_0R_1}}{n})+|(\epsilon'_2+2\epsilon'_3)\mathcal H_1(\tilde \Pi|\mathcal G_n)|+\epsilon'_3\mathcal H_2(\tilde \Pi|\mathcal G_n).
\end{align*}
\end{proof}

\textbf{Proof of Theorem \ref{thm.first/second fitler}.}
\begin{proof}
According to Theorem \ref{thm.low/high fitler}, for a given graph $\mathcal G_n$, for given filter $g_1$,  $g_2$ there exists $g_1'$ and $g_2'$ such that

\begin{align*}
\mathcal H_1(\tilde g'_1(I-\tilde \Pi)|\mathcal G_n)&=\frac a2(1-\frac {2\sqrt{R_0R_1}}{n})+|\epsilon_2\mathcal H_1(\tilde \Pi|\mathcal G_n)|\\
\mathcal H_1(\tilde g'_2(I-\tilde \Pi)|\mathcal G_n)&=(\frac b2-\frac 13\epsilon'_3)(1-\frac {2\sqrt{R_0R_1}}{n})+|(\epsilon'_2+2\epsilon'_3)\mathcal H_1(\tilde \Pi|\mathcal G_n)|+\epsilon'_3\mathcal H_2(\tilde \Pi|\mathcal G_n).
\end{align*}

Since $\forall g_1\in\mathcal S_g^1,g_2\in\mathcal S_g^2$, $g_1([0,2])\in[0,1]$ and $g_2([0,2])\in[0,1]$, then we have 
$|\epsilon_2|\le \frac 12$ and $|\epsilon_2+2\epsilon_3|\le \frac 12$.  
\begin{align*}
\mathcal H_1(\tilde g'_1(I-\tilde \Pi)|\mathcal G_n)&\le\frac 12(1-\frac {2\sqrt{R_0R_1}}{n})+\frac 12|\mathcal H_1(\tilde \Pi|\mathcal G_n)|\\
\mathcal H_1(\tilde g'_2(I-\tilde \Pi)|\mathcal G_n)&\le(\frac 12-\frac 23\epsilon'_3)(1-\frac {2\sqrt{R_0R_1}}{n})+\frac 12|\mathcal H_1(\tilde \Pi|\mathcal G_n)|+\epsilon'_3\mathcal H_2(\tilde \Pi|\mathcal G_n).
\end{align*}
These inequalities hold when $\epsilon_1=1,|\epsilon_2|= \frac 12$ and $\epsilon|\epsilon_2+2\epsilon_3|= \frac 12$.  Let $g_1$ with $|\epsilon_2|= \frac 12$ and $g_2$ with $|\epsilon_2+2\epsilon_3|= \frac 12$, then we have 
\begin{align*}
&\mathcal H_1(\tilde g'_2(I-\tilde \Pi)|\mathcal G_n)-\mathcal H_1(\tilde g'_1(I-\tilde \Pi)|\mathcal G_n)\\
=&(-\frac 23\epsilon'_3)(1-\frac {2\sqrt{R_0R_1}}{n})+\epsilon'_3\mathcal H_2(\tilde \Pi|\mathcal G_n)\\
=&-\frac {\epsilon_3'}n(R_0(\frac 23-\tilde \pi_0^2 )+R_1(\frac 23-\tilde \pi_1^2)+2\sqrt{R_0R_1}(\frac 23-\tilde \Pi_{01}^2))\\
\ge&\frac {\epsilon'_3}3\mathcal H_2(\tilde \Pi|\mathcal G_n)
\end{align*}
Therefore,  $\mathcal H_1(\tilde g'_2(I-\tilde \Pi)|\mathcal G_n)>\mathcal H_1(\tilde g'_1(I-\tilde \Pi)|\mathcal G_n)$ when $\mathcal H_2(\tilde \Pi|\mathcal G_n)<\frac 23(1-\frac {2\sqrt{R_0R_1}}{n})$, otherwise then we have $|\mathcal H_1(\tilde g'_2(I-\tilde \Pi)|\mathcal G_n)-\mathcal H_1(\tilde g'_1(I-\tilde \Pi)|\mathcal G_n)|\le -\frac {\epsilon'_3}3\mathcal H_2(\tilde \Pi|\mathcal G_n)$. 
\end{proof}

\subsection{Proof of Theorem \ref{thm.filter bank1}}\label{Ap.proof filter bank}
\begin{lemma}\label{lm.filter bank}
For any function $g(\cdot)$ which is nonnegative on the closed interval $[0,2]$ , 
$$G(g, \tilde \Pi)=R_0(\tilde g(I-\tilde \Pi))_{00}+R_1\tilde g(I-\tilde \Pi))_{11}-2\sqrt {R_0R_1}\tilde g(I-\tilde \Pi))_{01}> 0.$$
\end{lemma}

\begin{proof}

\begin{align*}
&R_0(\tilde g(I-\tilde \Pi))_{00}+R_1\tilde g(I-\tilde \Pi))_{11}-2\sqrt {R_0R_1}\tilde g(I-\tilde \Pi))_{01}\\
=&\mathbf y_0^\top g(\tilde L)\mathbf y_0+\mathbf y_1^\top g(\tilde L)\mathbf y_1-2\mathbf y_0^\top g(\tilde L)\mathbf y_1\\
=&(\mathbf y_0-\mathbf y_1)^\top g(\tilde L)(\mathbf y_0-\mathbf y_1)\\
\end{align*}

Since $\tilde L$ is symmetric positive definite with eigenvalues falling in $[0,2]$, for any nonnegative $g$, $g(\tilde L)$ is also symmetric positive definite. Thus, $(\mathbf y_0-\mathbf y_1)^\top g(\tilde L)(\mathbf y_0-\mathbf y_1)>0$.
\end{proof}

\textbf{Proof of Theorem \ref{thm.filter bank1}}
\begin{proof}

Here we assume that $\mathcal I_{\gamma-\omega}=\emptyset$, that is for any $g$,  $\mathcal I_{g,\gamma-\omega}=\mathcal I_{g}$. Below, we assume that $R_0\ne R_1$.

It is easy to check that for any $g_1(\cdot),g_2(\cdot)\in\mathcal S_g$ and nonnegative constant $l_1,l_2$ with $l_1+l_2=1$, $(l_1g_1+l_2g_2)(\cdot)\in\mathcal S_g$. 

For any $g\in\mathcal S_g$, $\tilde g(I-\tilde \Pi)=(\epsilon_1+\epsilon_2)I-\epsilon_2\tilde \Pi$, thus

\begin{align*}
\frac {G(g, \tilde \Pi)}{\sum_ig(\lambda_i)}&=\frac {R_0(\tilde g(I-\tilde \Pi))_{00}+R_1\tilde g(I-\tilde \Pi))_{11}-2\sqrt {R_0R_1}\tilde g(I-\tilde \Pi))_{01}}{\sum_i(\epsilon_1+\epsilon_2\lambda_i)}\\
&=\frac {(\epsilon_1+\epsilon_2)(\sqrt {R_0}-\sqrt {R_1})^2-\epsilon_2(R_0\tilde \pi_0+R_1\tilde \pi_1-2\sqrt {R_0R_1}\tilde \Pi_{01})}{\epsilon_1n+\epsilon_2\sum_i\lambda_i}
\end{align*}

Denote $A(\tilde \Pi)=(\sqrt {R_0}-\sqrt {R_1})^2\ge 0$, $B(\tilde \Pi)=R_0\tilde \pi_0+R_1\tilde \pi_1-2\sqrt {R_0R_1}\tilde \Pi_{01}$, $C(\tilde \Pi)=\frac {A(\tilde \Pi)-B(\tilde \Pi)}{A(\tilde \Pi)}$ and $\bar \lambda =\frac {\sum_i\lambda_i}{n}$, then
$$
\frac {G(g, \tilde \Pi)}{n\sum_ig(\lambda_i)}=\frac {(\epsilon_1+\epsilon_2)A(\tilde \Pi)-\epsilon_2B(\tilde \Pi)}{(\epsilon_1+\epsilon_2\bar \lambda)n^2}=\frac {(\epsilon_1+\epsilon_2C(\tilde \Pi))A}{(\epsilon_1+\epsilon_2\bar \lambda)n^2}.
$$
If $\epsilon_2>0$, $g$ is a high-pass filter and $m_g=\epsilon_1$; if $\epsilon_2<0$, $g$ is a low-pass filter and $m_g=\epsilon_1+\epsilon_2\lambda_{n-1}$ where $\lambda_{n-1}$ is the maximal eigenvalue.

$$
H(\epsilon_1,\epsilon_2)=m_g\log\frac {G(g, \tilde \Pi)}{n\sum_ig(\lambda_i)}= 
\begin{dcases}
      \epsilon_1\log\frac {(\epsilon_1+\epsilon_2C(\tilde \Pi))A}{(\epsilon_1+\epsilon_2\bar \lambda)n^2} & \epsilon_1\in(0,1), \epsilon_2\in(0,\frac {1-\epsilon_1}2) \\
      (\epsilon_1+\epsilon_2\lambda_{n-1})\log\frac {(\epsilon_1+\epsilon_2C(\tilde \Pi))A}{(\epsilon_1+\epsilon_2\bar \lambda)n^2} & \epsilon_1\in(0,1), \epsilon_2\in(\frac {-\epsilon_1}2,0)  \\
\end{dcases}
$$
Let $f(\epsilon_1, \epsilon_2)=\epsilon_1\log\frac {(\epsilon_1+\epsilon_2C(\tilde \Pi))A}{(\epsilon_1+\epsilon_2\bar \lambda)n^2}$ and $h(\epsilon_1,\epsilon_2)=\epsilon_2\log\frac {(\epsilon_1+\epsilon_2C(\tilde \Pi))A}{(\epsilon_1+\epsilon_2\bar \lambda)n^2}$, then we have:

\begin{align*}
&\frac {\partial \frac f{\epsilon_1}}{\partial\epsilon_1}=\frac{1}{\epsilon_1+C(\tilde \Pi)\epsilon_2}-\frac{1}{\epsilon_1+\bar\lambda\epsilon_2}, \quad \frac {\partial \frac f{\epsilon_1}}{\partial\epsilon_2}=\frac{C(\tilde \Pi)}{\epsilon_1+C(\tilde \Pi)\epsilon_2}-\frac{\bar\lambda}{\epsilon_1+\bar\lambda\epsilon_2},\\
&\frac {\partial^2f}{\partial\epsilon_1^2}=\frac{\epsilon_1+2C(\tilde \Pi)\epsilon_2}{(\epsilon_1+C(\tilde \Pi)\epsilon_2)^2}-\frac{\epsilon_1+2\bar\lambda\epsilon_2}{(\epsilon_1+\bar\lambda\epsilon_2)^2},\quad  \frac {\partial^2f}{\partial\epsilon_2^2}=\frac{\bar\lambda^2\epsilon_1}{(\epsilon_1+\bar\lambda\epsilon_2)^2}-\frac{\big(C(\tilde \Pi)\big)^2\epsilon_1}{(\epsilon_1+C(\tilde \Pi)\epsilon_2)^2},\\
&\frac {\partial^2f}{\partial\epsilon_1\epsilon_2}=\frac{\big(C(\tilde \Pi)\big)^2\epsilon_2}{(\epsilon_1+C(\tilde \Pi)\epsilon_2)^2}-\frac{\bar\lambda^2\epsilon_2}{(\epsilon_1+\bar\lambda\epsilon_2)^2},\quad  \frac {\partial^2h}{\partial\epsilon_1^2}=\frac{\epsilon_2}{(\epsilon_1+\bar\lambda\epsilon_2)^2}-\frac{\epsilon_2}{(\epsilon_1+C(\tilde \Pi)\epsilon_2)^2},\\
&\frac {\partial^2h}{\partial\epsilon_2^2}=C(\tilde \Pi)\frac{2\epsilon_1+C(\tilde \Pi)\epsilon_2}{(\epsilon_1+C(\tilde \Pi)\epsilon_2)^2}-\bar\lambda\frac{2\epsilon_1+\bar\lambda\epsilon_2}{(\epsilon_1+\bar\lambda\epsilon_2)^2},\quad  \frac {\partial^2h}{\partial\epsilon_1\epsilon_2}=\frac{\epsilon_1}{(\epsilon_1+C(\tilde \Pi)\epsilon_2)^2}-\frac{\epsilon_1}{(\epsilon_1+\bar\lambda\epsilon_2)^2}
\end{align*}

According to Lemma \ref{lm.filter bank}, 

\begin{align*}
\frac{A(\tilde \Pi)-B(\tilde \Pi)}{A(\tilde \Pi)}&=\frac {(\sqrt {R_0}-\sqrt {R_1})^2-R_0\tilde \pi_0-R_1\tilde \pi_1+2\sqrt {R_0R_1}\tilde \Pi_{01}}{(\sqrt {R_0}-\sqrt {R_1})^2}\\
&=\frac {R_0(1-\tilde \pi_0)+R_1(1-\tilde \pi_1)-2\sqrt {R_0R_1}(1-\tilde \Pi_{01})}{(\sqrt {R_0}-\sqrt {R_1})^2}>0
\end{align*}

When $C(\tilde\Pi)\le\bar\lambda$, $\epsilon_2<0$ is better.

I. $A(\tilde \Pi)\ne0$ and $\epsilon_2 > 0 $

The Hessian determinant of $H(\epsilon_1,\epsilon_2)$ is :
$$
Hess(H)=\frac{\bar\lambda^2\epsilon_1^2(1-\bar\lambda)}{(\epsilon_1+\bar\lambda\epsilon_2)^4}+\frac{\big(C(\tilde \Pi)\big)^2\epsilon_1^2(1-C(\tilde \Pi))}{(\epsilon_1+C(\tilde \Pi)\epsilon_2)^4}.
$$
When $B(\tilde \Pi)<0$, $C(\tilde \Pi)<1$, then $Hess(H)>0$, $H(\epsilon_1,\epsilon_2)$ is a convex function; 

when $B(\tilde \Pi)>0$, $C(\tilde \Pi)>1>\bar\lambda$, then $Hess(H)$ can be positive  or negative.
\begin{align*}
&\begin{dcases}
\frac {\partial H}{\partial \epsilon_1}=\log\frac {(\epsilon_1+\epsilon_2C(\tilde \Pi))A}{(\epsilon_1+\epsilon_2\bar \lambda)n^2}+\frac{\epsilon_1}{\epsilon_1+C(\tilde \Pi)\epsilon_2}-\frac{\epsilon_1}{\epsilon_1+\bar\lambda\epsilon_2},\\
\frac {\partial H}{\partial \epsilon_2}=\frac{C(\tilde \Pi)\epsilon_1}{\epsilon_1+C(\tilde \Pi)\epsilon_2}-\frac{\bar\lambda\epsilon_1}{\epsilon_1+\bar\lambda\epsilon_2},\\
\end{dcases}
\end{align*}
Since $A(\Pi)=R_0+R_1-2\sqrt{R_0R_1}\le n$, therefore $H(\epsilon_1,\epsilon_2)$ has no stationary point. 

\begin{itemize}
   
 \item a. When $C(\tilde \Pi)>1>\bar\lambda$, $\frac {\partial H}{\partial \epsilon_2}>0$, then fix $\epsilon_1$, $H$ is monotonic increasing function w.r.t $\epsilon_2\in(0,\frac {1-\epsilon_1}2)$.

In this case, for any $g_1(\tilde L)=\epsilon_1+\epsilon_2$, there exist $\epsilon_2'\in (\epsilon_2,\frac {1-\epsilon_1}2)$ such that for any positive constant $l$, $$H(\epsilon_1,l\epsilon_2+(1-l)\epsilon_2')>H(\epsilon_1\epsilon_2).$$

\item b. When $C(\tilde \Pi)<\bar\lambda$, $\frac {\partial H}{\partial \epsilon_2}<0$, then fix $\epsilon_1$, $H$ is monotonic decreasing function w.r.t $\epsilon_2\in(0,\frac {1-\epsilon_1}2)$.

In this case, for any $g_1(\tilde L)=\epsilon_1+\epsilon_2$, there exist $\epsilon_2'\in (0,\epsilon_2)$ such that for any positive constant $l$, $$H(\epsilon_1,l\epsilon_2+(1-l)\epsilon_2')>H(\epsilon_1\epsilon_2).$$

\item c. When $C(\tilde \Pi)=\bar\lambda$,  $\frac {\partial H}{\partial \epsilon_1}=\log\frac A{n^2}<0$, then fix $\epsilon_2$, $H$ is monotonic decreasing function w.r.t $\epsilon_1\in(0,1]$. 

In this case, for any $g_1(\tilde L)=\epsilon_1+\epsilon_2$, there exist $\epsilon_1'\in (0,\epsilon_1)$ such that for any positive constant $l$, $$H(l\epsilon_1+(1-l)\epsilon_1',\epsilon_2)>H(\epsilon_1\epsilon_2).$$

\end{itemize}

II.$A(\tilde \Pi)\ne0$ and $\epsilon_2<0$

Since,
\begin{align*}
&\begin{dcases}
\frac {\partial H}{\partial \epsilon_1}=\log\frac {(\epsilon_1+\epsilon_2C(\tilde \Pi))A}{(\epsilon_1+\epsilon_2\bar \lambda)n^2}+\frac{\epsilon_1+\epsilon_2\lambda_{n-1}}{\epsilon_1+C(\tilde \Pi)\epsilon_2}-\frac{\epsilon_1+\epsilon_2\lambda_{n-1}}{\epsilon_1+\bar\lambda\epsilon_2},\\
\frac {\partial H}{\partial \epsilon_2}=\lambda_{n-1}\log\frac {(\epsilon_1+\epsilon_2C(\tilde \Pi))A}{(\epsilon_1+\epsilon_2\bar \lambda)n^2}+\frac{C(\tilde \Pi)(\epsilon_1+\epsilon_2\lambda_{n-1})}{\epsilon_1+C(\tilde \Pi)\epsilon_2}-\frac{\bar\lambda(\epsilon_1+\epsilon_2\lambda_{n-1})}{\epsilon_1+\bar\lambda\epsilon_2},\\
\end{dcases}
\end{align*}
Similarly, $H(\epsilon_1,\epsilon_2)$ has no stationary point.
\begin{itemize}
    \item 
a. When $C(\tilde \Pi)>1>\bar\lambda$, $\frac {\partial H}{\partial \epsilon_1}<0$, then fix $\epsilon_2$, $H$ is monotonic decreasing function w.r.t $\epsilon_1\in(0,1]$. 

In this case, for any $g_1(\tilde L)=\epsilon_1+\epsilon_2$, there exist $\epsilon_1'\in (0,\epsilon_1)$ such that for any positive constant $l$, $$H(l\epsilon_1+(1-l)\epsilon_1',\epsilon_2)>H(\epsilon_1\epsilon_2).$$

\item b. When $C(\tilde \Pi)\le\bar\lambda$, $\frac {\partial H}{\partial \epsilon_2}<0$, then fix $\epsilon_1$, $H$ is monotonic decreasing function w.r.t $\epsilon_2\in(0,\frac {1-\epsilon_1}2)$.

In this case, for any $g_1(\tilde L)=\epsilon_1+\epsilon_2$, there exist $\epsilon_2'\in (0,\epsilon_2)$ such that for any positive constant $l$, $$H(\epsilon_1,l\epsilon_2+(1-l)\epsilon_2')>H(\epsilon_1\epsilon_2).$$

\end{itemize}

III. $A(\tilde \Pi)\ne0$ 

When $A(\tilde \Pi)=0$, since $A(\tilde\Pi)-B(\tilde \Pi)>0$, then $B(\tilde \Pi)<0$,
\begin{align*}
H(\epsilon_1,\epsilon_2)=m_g\log\frac {G(g, \tilde \Pi)}{n\sum_ig(\lambda_i)}= \begin{dcases}
      \epsilon_1\log\frac {-\epsilon_2B(\tilde \Pi)}{(\epsilon_1+\epsilon_2\bar \lambda)n^2} & \epsilon_1\in(0,1], \epsilon_2\in(0,\frac {1-\epsilon_1}2] \\
      (\epsilon_1+\epsilon_2\lambda_{n-1})\log\frac {-\epsilon_2B(\tilde \Pi)}{(\epsilon_1+\epsilon_2\bar \lambda)n^2} & \epsilon_1\in(0,1], \epsilon_2\in(\frac {-\epsilon_1}2,0)  \\
\end{dcases}
\end{align*}

\begin{itemize}
   
 \item a. $\epsilon_2>0$

\begin{align*}
&\begin{dcases}
\frac {\partial H}{\partial \epsilon_1}=\log\frac {-\epsilon_2B(\tilde \Pi)}{(\epsilon_1+\epsilon_2\bar \lambda)n^2}-\frac{\epsilon_1}{\epsilon_1+\bar\lambda\epsilon_2}<0,\\
\frac {\partial H}{\partial \epsilon_2}=\frac{\epsilon_1}{\epsilon_2}-\frac{\bar\lambda\epsilon_1}{\epsilon_1+\bar\lambda\epsilon_2}>0,\\
\end{dcases}
\end{align*}

Fix $\epsilon_1$, $H$ is monotonic increasing function w.r.t $\epsilon_2\in(0,\frac {1-\epsilon_1}2)$. In this case, for any $g_1(\tilde L)=\epsilon_1+\epsilon_2$, there exist $\epsilon_2'\in (\epsilon_2,\frac {1-\epsilon_1}2)$ such that for any positive constant $l$, $H(\epsilon_1,l\epsilon_2+(1-l)\epsilon_2')>H(\epsilon_1\epsilon_2)$ .

\item b. $\epsilon_2<0$

  \begin{align*}
  &\begin{dcases}
  \frac {\partial H}{\partial \epsilon_1}=\log\frac {-\epsilon_2B(\tilde \Pi)}{(\epsilon_1+\epsilon_2\bar \lambda)n^2}-\frac{\epsilon_1+\epsilon_2\lambda_{n-1}}{\epsilon_1+\bar\lambda\epsilon_2}<0,\\
  \frac {\partial H}{\partial \epsilon_2}=\lambda_{n-1}\log\frac {-\epsilon_2B(\tilde \Pi)}{(\epsilon_1+\epsilon_2\bar \lambda)n^2}+\frac{\epsilon_1+\lambda_{n-1}}{\epsilon_2}-\frac{\bar\lambda(\epsilon_1+\epsilon_2\lambda_{n-1})}{\epsilon_1+\bar\lambda\epsilon_2},\\
  \end{dcases}
  \end{align*}
 
  Fix $\epsilon_2$, $H$ is monotonic decreasing function w.r.t $\epsilon_1\in(0,1]$. In this case, for any $g_1(\tilde L)=\epsilon_1+\epsilon_2$, there exist $\epsilon_1'\in (0,\epsilon_1)$ such that for any positive constant $l$, $H(l\epsilon_1+(1-l)\epsilon_1',\epsilon_2)>H(\epsilon_1,\epsilon_2)$ .
 
\end{itemize}
\end{proof}
\end{document}